\documentclass[sigconf]{aamas}

\usepackage{preamble}

\newif\ifarxiv
\arxivtrue 

\begin{document}


\pagestyle{fancy}
\fancyhead{}




\acmSubmissionID{973}


\title[Hierarchical Imitation Learning of Team Behavior from Heterogeneous Demonstrations]{Hierarchical Imitation Learning of Team Behavior from Heterogeneous Demonstrations}


\author{Sangwon Seo}
\affiliation{
  \institution{Rice University}
  \city{Houston, TX}
  \country{USA}}
\email{sangwon.seo@rice.edu}

\author{Vaibhav Unhelkar}
\affiliation{
  \institution{Rice University}
  \city{Houston, TX}
  \country{USA}}
\email{vaibhav.unhelkar@rice.edu}


\begin{abstract}
Successful collaboration requires team members to stay aligned, especially in complex sequential tasks. Team members must dynamically coordinate which subtasks to perform and in what order. However, real-world constraints like partial observability and limited communication bandwidth often lead to suboptimal collaboration. Even among expert teams, the same task can be executed in multiple ways. To develop multi-agent systems and human-AI teams for such tasks, we are interested in data-driven learning of multimodal team behaviors. Multi-Agent Imitation Learning (MAIL) provides a promising framework for data-driven learning of team behavior from demonstrations, but existing methods struggle with heterogeneous demonstrations, as they assume that all demonstrations originate from a single team policy.
Hence, in this work, we introduce \ouralg: a hierarchical MAIL algorithm designed to learn multimodal team behaviors in complex sequential tasks. \ouralg represents each team member with a hierarchical policy and learns these policies from heterogeneous team demonstrations in a factored manner. By employing a distribution-matching approach, \ouralg mitigates compounding errors and scales effectively to long horizons and continuous state representations. Experimental results show that \ouralg outperforms MAIL baselines and accurately models team behavior across a variety of collaborative scenarios.
\end{abstract}


\maketitle

\section{Introduction}
\label{sec. intro}
Imitation learning (IL) is a paradigm for training agent behaviors using demonstrations \cite{abbeel2004apprenticeship}. IL typically assumes that the demonstrations are generated by an expert following a single, optimal policy. Under this assumption, IL algorithms learn an estimate of the expert's policy. Compared to reinforcement learning (RL), another popular framework for training agents, IL offers a key advantage in many practical applications: it does not require hand-engineered rewards. Designing rewards often requires extensive domain expertise, and in many cases, informative rewards can be challenging for end-users to engineer. Moreover, even with hand-engineered rewards, the agent can learn suboptimal or reward hacking behaviors~\cite{amodei2016concrete, sutton2018reinforcement}. In contrast, often end-users can more easily provide demonstrations of desirable agent behavior~\cite{osa2018algorithmic, arora2021survey, chernova2022robot}.

Despite this advantage, traditional IL algorithms face challenges while learning complex behaviors from demonstration collected by human end-users. A key challenge is that conventional IL methods consider a \textit{single agent} that has \textit{full observability} of the environment. In reality, human end-users often perform complex tasks as teams rather than individually. As a result, demonstrations available for learning often involve multiple agents interacting with one another and their environment. Since the dynamics of the environment depend on these interactions, specialized IL algorithms are needed to learn multi-agent behaviors.

\begin{figure*}[t]
\newcommand\gap{0.25}
\newcommand\gapd{0.48}
  \centering
  \begin{subfigure}[b]{\gapd\linewidth}
      \centering
      \includegraphics[width=0.49\textwidth]{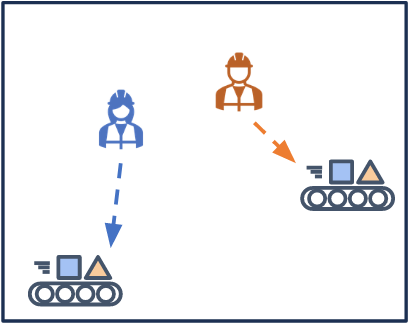}
      \hfill
      \includegraphics[width=0.49\textwidth]{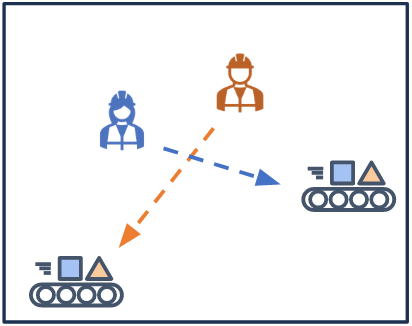}
      \caption{Multiple Near-Optimal Strategies}
      \label{fig: workplace 1}
  \end{subfigure}
  \hfill  
  \begin{subfigure}[b]{\gap\linewidth}
      \centering
      \includegraphics[width=0.95\textwidth]{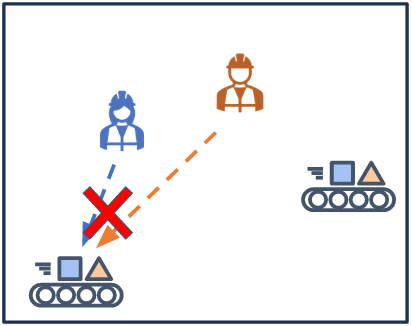}
      \caption{Suboptimal Teamwork}
      \label{fig: workplace suboptimal}
  \end{subfigure}
  \hfill  
  \begin{subfigure}[b]{\gap\linewidth}
      \centering
      \includegraphics[width=0.95\textwidth]{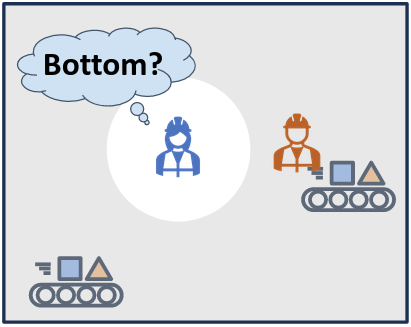}
      \caption{Partial Observability}
      \label{fig: workplace partially observe}
  \end{subfigure}

  \captionsetup{subrefformat=parens}
  \caption{Motivating Example: Consider a team whose members must coordinate on the fly to complete subtasks at two conveyor belts. Each member has limited observability, perceiving only their immediate surroundings. For example, the unshaded area for the blue person in \subref{fig: workplace partially observe}. As shown in \subref{fig: workplace 1}, this task allows multiple near-optimal strategies, enabling teams to execute it in different ways based on their shared preferences. However, practical constraints -- such as partial observability -- can lead to suboptimal coordination and team performance. For instance, if multiple members gather at the same subtask location, it results in inefficient task allocation, where one subtask remains unattended while two members redundantly perform the same task \subref{fig: workplace suboptimal}. Like many real-world scenarios, this task engenders heterogeneous and potentially suboptimal demonstrations of teamwork. This paper focuses on learning models of team behavior in this challenging setting from demonstrations.}
  \label{fig: workplace}
\end{figure*}

To address this challenge, recent approaches have extended imitation learning to multi-agent settings \cite{song2018multi,yu2019multi,bhattacharyya2018multi,yang2020bayesian}. These \textit{Multi-agent IL} (MAIL) methods typically assume that demonstrations are generated by a single, well-defined multi-agent policy \cite{lin2019multi}. However, due to practical challenges, this assumption is difficult to satisfy in complex multi-agent tasks encountered in the real world.
As illustrated in Fig. \ref{fig: workplace 1}, complex multi-agent tasks often consist of multiple subtasks. Demonstrations of such tasks frequently involve a variety of subtask allocations, including suboptimal ones. Suboptimality can arise from various factors, including the decentralized nature of multi-agent task execution and the partial observability that individual agents have of the task environment and other agents \cite{reyes2019makes, seo2021towards, seo2025socratic}.
Consequently, real-world demonstrations inherently exhibit multiple modes of multi-agent behavior, diverging from the assumptions of most existing MAIL methods.
\ifarxiv
\blfootnote{This article is an extended version of an identically-titled paper accepted at the \textit{International Conference on Autonomous Agents and Multiagent Systems (AAMAS 2025)}.}
\fi

In single-agent settings, \textit{hierarchical imitation learning} methods that explicitly impose a hierarchical structure on an expert's decision-making have been employed to address the challenge of modeling multimodal behaviors \cite{unhelkar2019learning,jing2021adversarial,sharma2018directed,seo2024idil,orlov2022factorial, jain2024godice}. However, little work has been done to extend these methods to multi-agent settings. To our knowledge, \textit{Bayesian Team Imitation Learner} (\btil) is the only approach that applies hierarchical imitation learning in a multi-agent context \cite{seo2022semi}. However, \btil is built on variational inference and tabular representations, making it difficult to scale to tasks with high-dimensional states and long horizons.

To enable MAIL for more complex tasks, this paper introduces \textit{Deep Team Imitation Learner} (\ouralg), a multi-agent hierarchical imitation learning algorithm. \ouralg rigorously extends single-agent hierarchical imitation learning to collaborative tasks conducted in partially observable environments, enabling the learning of multimodal team behaviors from heterogeneous demonstrations, even in tasks with long horizons and continuous state representations.
At its core, \ouralg leverages the state-action distribution matching framework, a mainstay of state-of-the-art IL methods due to its performance and scalability \cite{ho2016generative,garg2021iq}. 

While recent hierarchical IL methods have applied distribution matching in single-agent settings \cite{jing2021adversarial,seo2024idil}, its extension to learning multimodal team behavior under partial observability remains unexplored. Notably, key theoretical results in distribution-matching-based hierarchical imitation learning, such as Theorem 1 (Bijection) in \cite{jing2021adversarial} and Theorem 2.4 (Convergence) in \cite{seo2024idil}, have only been proven under full observability. Thus, additional theoretical justification is required for learning multimodal team behavior in partially observable settings.
Hence, we first extend these theoretical results to the partially observable multi-agent hierarchical imitation learning setting. Next, leveraging these theoretical results, we derive \ouralg to effectively learn the hierarchical team policies. Finally, we evaluate \ouralg on a suite of collaborative tasks, demonstrating that it outperforms MAIL baselines in modeling team behavior across multiple scenarios.
\section{Related Works}
We begin with a brief overview of related research.

\subsection{Imitation Learning of Multimodal Behavior}
Extensive research has been conducted on learning multimodal behaviors from demonstrations. In works such as \cite{li2017infogail, hausman2017multi}, the authors extend Generative Adversarial Imitation Learning (GAIL) to learn a policy that depends on a learned latent state. This learned latent state effectively encodes different modes of the behavior. However, these methods assume the latent states remain static during a task execution; thus, their methods are unsuitable for modeling agent behavior whose latent states can change during the tasks. 
Other approaches, such as \cite{schmerling2018multimodal}, use Conditional Variational Autoencoders (CVAE) to capture multimodal human behavior, but the learned latent space responsible for generating multimodality lacks grounding and difficult to associate with specific subtasks. 

Informed by the Option framework \cite{sutton1999between}, another line of research leverages hierarchical policies to model multimodal behavior. Hierarchical policies typically consider two levels: high-level policies that govern decision-making over extended temporal intervals, and low-level policies responsible for executing specific actions within shorter time frames \cite{le2018hierarchical, jing2021adversarial}. 
To learn such policies from demonstrations, various approaches have been explored; e.g., variational inference \cite{unhelkar2019learning, orlov2022factorial}, hierarchical behavior cloning \cite{le2018hierarchical, kipf2019compile, zhang2021provable}, and hierarchical variants of GAIL \cite{sharma2018directed, lee2020learning, jing2021adversarial, chen2023multi}. Most recently, \cite{seo2024idil} propose a factored distribution-matching approach to train hierarchical policies. While all these methods show remarkable performance in single-agent tasks, their extension to multi-agent scenarios has been rarely explored and often lacks theoretical grounds.

\subsection{Multi-agent Imitation Learning}
Learning team behavior from demonstration can be framed as a multi-agent imitation learning problem. 
Since \cite{song2018multi} introduced the multi-agent variant of generative adversarial imitation learning (called MA-GAIL), several extensions have been proposed to enhance its training efficiency and scalability \cite{yu2019multi, liu2020multi, yang2020bayesian, bhattacharyya2018multi, bhattacharyya2019simulating, sengadu2023dec}.
However, these methods generally assume that the demonstrations originate from a single multi-agent policy, limiting their ability to capture diverse team behaviors. 
Despite the importance of accounting for multimodality when modeling team behavior, only a few approaches have incorporated latent states into MAIL.
Among these, \cite{le2017coordinated} model agent roles as latent variables, while \cite{wang2022co} represent strategies as latent features. 
However, both methods assume static latent states and do not consider their dynamics. 
To address this gap, \cite{seo2022semi} propose Bayesian Team Imitation Learner (\btil), a multi-agent extension of \cite{unhelkar2019learning}, which can learn hierarchical policies of all team members from team demonstrations. Nonetheless, \btil struggles with large, complex tasks and suffers from compounding errors. In contrast, \ouralg overcomes these limitations by utilizing function approximators (e.g., neural networks) and augmenting demonstrations with online samples collected during training.

\section{Background}
\label{sec. background}
In this section, we present preliminaries on distribution-matching-based imitation learning and introduce the mathematical model of team behavior.

\subsection{Imitation Learning via\\ Distribution Matching}
\label{sec. distribution matching}
Using the Markov Decision Process (MDP) framework, an agent's behavior is defined by a policy $\pi(a|s)$, which represents the probability distribution of an action $a$ given a state $s$. The goal of imitation learning is to minimize the discrepancy (represented as a loss $L$) between the learner's policy $\pi$ and the expert's policy $\pi_E$: $\min_{\pi} L(\pi, \pi_E)$. However, due to the inaccessibility of $\pi_E$, this objective is often ill-defined and highly challenging to solve.

To address this, \citet{ho2016generative} reformulate imitation learning as a problem of matching the occupancy measures of the learner and the expert. The (normalized) occupancy measure of a policy $\pi$ is defined as $\OMsa{\pi}{} \doteq \sumP{\TmS{t}\myeq s, \TmA{}{t}\myeq a|\pi}$, implying the stationary distribution over states $s$ and actions $a$ induced by $\pi$. Thanks to the one-to-one correspondence between a policy $\pi(s|a)$ and its occupancy measure $\OMsa{\pi}{}$ \cite{puterman2014markov}, matching the occupancy measures is equivalent to matching the policies. This can be formalized as:
\begin{align*}
    \argmin_{\pi}\Dfsa{\pi}{\piE}{} 
\end{align*}
where $\rhopi$ is the learner's occupancy measure, $\rhoE$ is the expert's occupancy measure, and $D_f$ denotes the $f$-divergence \cite{ghasemipour2020divergence}. While direct access to $\rhoE$ is still infeasible, it can be approximated using the empirical distribution calculated from expert demonstrations $\DemoTeam$. 
Due to its performance and scalability, since its introduction, the distribution matching approach has become a mainstream technique in imitation learning, giving rise to numerous variants, including the following multi-agent and hierarchical ones.

\paragraph{Multi-Agent Variants.} 
Assuming a unique equilibrium in multi-agent behaviors, \citet{song2018multi} formulate an occupancy measure matching problem in multi-agent settings:
\begin{align}
   \argmin_{\pi}\sum_{i=1}^n\Dfsa{\pi_i,\OtPi{i}}{\pi_{E}}{i}  \label{eq. MAIL objective}
\end{align}
where $\OtPi{i}$ is joint policies except the $i$-th agent's policy, $\pi_E$ denotes multi-agent expert policy at equilibrium, and $\OMsa{\pi_i,\OtPi{i}}{i} \doteq \sumP{\TmS{t}\myeq s, \TmA{i}{t}\myeq a_i|\pi_i, \OtPi{i}}$. This objective function implies that we can iteratively minimize the objective with respect to individual policies $\pi_1, \cdots, \pi_n$, and the updates can be calculated similarly to the single-agent problem by considering other agents' policies as part of the environment dynamics. 

\paragraph{Hierarchical Variants.} 
\citet{jing2021adversarial} extend occupancy measure matching approach to hierarchical imitation learning.
They model an agent behavior as an option policy $\tilpi=(\pi_L, \pi_H)$, where $\pi_L(a|s, x)$ and $\pi_H(x|s, x^-)$ are referred to as a low- and high-level policies, respectively, with $x$ being an option. Additionally, they define an option-occupancy measure corresponding to $\tilpi$ as 
$$\OMsaxx{\tilpi} \doteq \sumP{\TmS{t}\myeq s, \TmA{}{t}\myeq a, \TmX{}{t}\myeq x, \TmX{}{t-1}\myeq \PrX{}|\tilpi}$$ 
and prove the one-to-one correspondence between  $\tilpi$ and $\rho_{\tilpi}$. Thus, hierarchical imitation learning can also be cast as distribution matching between two option-occupancy measures $\rho_\tilpi$ and $\rho_E$:
\begin{align}
   \argmin_{\tilpi}\Dfsaxx{\tilpi}{\piE}     \label{eq. HIL objective}
\end{align}

\subsection{Model of Team Behavior}
\label{sec. teamwork model}
We borrow a model of team behavior introduced in \cite{seo2023automated} to represent our multi-agent behavior in sequential team tasks. The model consists of a decentralized partially observable MDP (Dec-POMDP) to capture the task dynamics and an Agent Markov models (AMM) to represent agents' multimodal behavior~\cite{oliehoek2016concise, unhelkar2019learning}.

\paragraph{Dec-POMDP} Dec-POMDP is a probabilistic model representing the dynamics of the partially observable sequential multi-agent tasks. It is expressed as a tuple $\calM = (n, S,\times A_i, T, \mu_0, \times \Omega_i, \{O_i\}, \gamma)$, where $n$ is the number of agents, $S$ is the set of states $s$, $A_i$ is the set of the $i$-th agent's actions $a_i$, $\Omega_i$ is the set of the $i$-th agent's observations $\obs_i$, $T(\NxS|s, a)$ denotes the probability of a state $\NxS$ transitioning from a state $s$ and a joint action $a\myeq(a_1, \cdots, a_n)$, $\mu_0(s)$ is an initial state distribution, and $\gamma$ is a discount factor. 
We define an extended action set as $\ExtA_i=A_i \cup \{\#\}$, where the symbol $\#$ denotes ``Not Available''.
$O_i:S \times \ExtA \rightarrow \Omega_i$ is an observation function for the $i$-th agent, which maps a pair of state $s$ and previous joint action $\PrJoA$ to an individual observation $\obs_i \in \Omega_i$. 
The values of previous actions $\PrA{i}$ at time $t\myeq 0$ are set as $\#$.
We denote capital letters without subscripts as joint spaces or joint functions, e.g., $A=\times A_i$ for a joint action space and $O=\prod_i O_i$ for a joint observation function.

\paragraph{Agent Markov Model (AMM)} 
When faced with complex team tasks, humans typically break them down into subtasks and dynamically adjust their plan regarding which subtasks to perform and in what order. Once they decide on the next subtask, they execute the necessary actions to complete it. 
AMM is designed to account for this hierarchical human behavior, and thus, is equivalent to hierarchical policies of Option framework with a one-step option \cite{sutton1999between, jing2021adversarial}.
Given a task model $\calM$, AMM defines the behavior model of the $i$-th agent as a tuple $(X_i, \pi_i, \Tx_i; \calM)$, where $X_i$ is the set of the possible subtasks, $\pi_i(a_i|\obs_i, x_i)$ denotes a subtask-driven policy, and $\Tx_i(x_i|\obs_i, \PrX{i})$ is the probability of an agent choosing their next subtask $x_i$ based on an observation $\obs_i$ and the current subtask $\PrX{i}$\footnote{In the following sections, we will omit the subscript $i$ from the function inputs, e.g., $\pi_i(a|\obs, x) \doteq \pi_i(a_i|\obs_i, x_i)$, when it is clear the functions pertain to individual observations, actions, and subtasks.}.
Following \cite{jing2021adversarial}, we define the value of the previous subtask at time $t\myeq 0$ as $\#$ and express the initial distribution of subtasks as $\Tx_i(x_i|\obs_i,\PrX{i}\myeq\#)$. Similar to previous works  \cite{jing2021adversarial,seo2022semi}, we assume the set of possible subtasks, $X$, is finite and given as prior knowledge. We then represent the AMM for the $i$-th agent simply as $(\pi_i, \Tx_i)$, omitting the non-learnable components $X_i$ and $\calM$.

\section{Problem Formulation}
\label{sec. problem formulation}

While \citet{seo2023automated} emphasize the need for modeling team behavior under partial observability and propose a corresponding mathematical model, few multi-agent imitation learning methods have been developed to address such complex teamwork models. To our knowledge, \btil is the only approach to learning multi-model team behavior from demonstrations \cite{seo2022semi}. However, \btil does not account for partial observability, and its applicability is limited to small, discrete domains, as both high- and low-level policies are constrained to categorical distributions. Thus, a practical method for learning team behavior models that addresses multimodality, partial observability, and scalability is still lacking. To derive such a method, we first formulate the problem of multi-agent imitation learning from heterogeneous demonstrations.

\subsection{Formalizing Hierarchical Multi-Agent Distribution Matching}
\label{sec. problem statement}
Inspired by the recent success of distribution matching-based imitation learning, we aim to apply this method to learn the team behavior model. Similar to the hierarchical variants for fully observable single-agent scenarios in Sec. \ref{sec. distribution matching}, we define \AMoccu occupancy measure for the $i$-th agent given a task model $\calM$ and joint agent models $\JoAM$ as:  
\begin{align*}
    &\OMam{\calN_i, \OtAM{i}}{i} \\
    &\hspace{3ex}\doteq \sumP{\TmO{i}{t}\myeq o_i, \TmA{i}{t}\myeq a_i, \TmX{i}{t}\myeq x_i, \TmX{i}{t-1}\myeq \PrX{i}|\JoAM, \calM}
\end{align*}
The notation $\rho_{\calN_i, \OtAM{i}}$, borrowed from \magail \cite{song2018multi}, represents the occupancy measure induced by the agent $i$'s policy $\calN_i$ and other agents' policy $\OtAM{i}$. 
Unless ambiguous, we simply denote $\rho_{i} \doteq \rho_{\calN_i, \OtAM{i}}$.

By combining this occupancy measure with the multi-agent variant of distribution matching introduced in Sec. \ref{sec. distribution matching}, we can formulate the distribution-matching problem for team behavior with $n$ agents as follows:
\begin{align}
    \argmin_{\JoAM} \sum_{i=1}^n \Dfam{i}{E}{i}{}   \label{eq. original objective}
\end{align}
where $\rho_{E}$ denotes the \AMoccu occupancy measure of the expert team model $(\pi_E, \Tx_E)$. In Sec. \ref{sec. theoretical grounds}, we provide theoretical justification for using Eq. \ref{eq. original objective} as the imitation learning objective. While this extension seems natural, the theoretical results in the existing literature are insufficient to guarantee that occupancy measure matching is equivalent to policy matching in partially observable multi-agent scenarios.

Additionally, informed by \idil \cite{seo2024idil}, we aim to adopt a factored approach to minimize the objective above. This factored approach enables us to leverage existing non-adversarial imitation learning methods, such as \iql \cite{garg2021iq}, which demonstrate more stable training compared to generative adversarial approaches. However, since the theoretical foundations for factored distribution matching are also developed under the assumption of full observability, further theoretical analysis is necessary to ensure its applicability in partially observable multi-agent settings. We provide this analysis in Sec. \ref{sec. theoretical grounds}.

\subsection{Problem Statement}
Since we cannot know which subtasks each member has in mind at the time of task execution, demonstrations contain only observations and actions. We define the set of $d$ demonstrations as $\DemoTeam = \Set{\TrajTeam_m}{m=1}{d}$, where $\TrajTeam \myeq \TmOATrj{}{0:h}$ is a trajectory of a team's task execution. 
We denote an individual trajectory of the $i$-th agent and the set of them as $\Traj{i} = \TmOATrj{i}{0:h}$ and $\Demo{i} = \Set{\Traj{m, i}}{m=1}{d}$, respectively, adding a subscript $i$.
The sequence of expert's subtasks corresponding to the $m$-th demonstration is defined as $\XSeq{m} = (\TmX{m}{0:h})$. Since the labels of the subtasks are challenging to collect in practice, only a small portion of them (e.g., for $l (\leq d)$ demonstrations) are optionally available.
Thus, our goal is to learn agent models $\{(\pi, \Tx)\}_{1:n}$ that exhibit the behaviors of $n$ team members from the following inputs: a multi-agent task model $\calM$, the set of possible subtasks $X$, heterogeneous demonstrations $\DemoTeam$, and optionally partial labels of subtasks $\Set{\XSeq{m}}{m=1}{l}$.

\section{Learning Model of Teamwork via Factored Distribution Matching}
\label{sec. theoretical grounds}
As mentioned in Sec. \ref{sec. problem statement}, matching occupancy measures does not always guarantee matching the team behavior models unless a one-to-one correspondence is established between the agent model (i.e., partial observation-based hierarchical policy) $\calN_i = (\pi_i, \Tx_i)$ and its \AMoccu occupancy measure $\rho_{i}$. 
Therefore, we first present this one-to-one correspondence, which extends the \textit{Theorem 1} from \cite{jing2021adversarial} to multi-agent partially-observable settings.

\begin{theorem}
\label{thm. bijection}
For each agent $i$, given a multi-agent task model $\calM$ and other agents' models $\OtAM{i}$, suppose $\rho_i$ is the \AMoccu occupancy measure for a stationary agent model $\calN_i = (\pi_i, \Tx_i; \calM)$ where 
\begin{align*}
    \Tx_i(x|\obs, \PrX{}) = \frac{\sum_{a}\OMam{i}{}}{\sum_{a,x}\OMam{i}{}}, ~
    \pi_i(a|\obs, x) = \frac{\sum_{x^-}\OMam{i}{}}{\sum_{x^-, a}\OMam{i}{}}.
\end{align*}
Then, $\calN_i = (\pi_i, \Tx_i; \calM)$ is the only agent model whose \AMoccu occupancy measure is $\rho_i$.
\end{theorem}
This can be proved similarly to \textit{Theorem 1} of \cite{jing2021adversarial} after deriving the stationary distributions of policy $\tilde{\pi}(w|v)$ and state transition $\tilde{T}(v'|v, w)$, where $v\doteq(\obs, x^-)$ and $w\doteq(x, a)$.
The complete proof is provided in the Appendix.
This theorem implies that we can consider the imitation learning of agent models $\JoAM$ as matching the \AMoccu occupancy measures between $\rho_{\calN_i, E_{-i}}$ and $\rho_{E}$ for all $i$. Here, $\rho_{\calN_i, E_{-i}}$ denotes the \AMoccu occupancy measure induced by the $i$-th agent model $\calN_i$ with other agents' models given as  expert models $E_{-i} \doteq (\Others{\pi_E}{i}, \Others{\Tx_E}{i})$. Thus, we can factorize the occupancy measure matching of the joint team model as follows: $\argmin_{\JoAM}\sum_{i=1}^n \Dfdot{\calN_i, E_{-i}}{E}$. Due to the one-to-one correspondence, the following two problems lead to the same optimal solution, $\calN_E$:
\begin{align*}
    &\argmin_{\JoAM}\sum_{i=1}^n \Dfdot{\calN_i, E_{-i}}{E}  \\
    &= \argmin_{\JoAM} \sum_{i=1}^n \Dfdot{\calN_i, \calN_{-i}}{E}  = \calN_E.
\end{align*}
This justifies our objective function, Eq \ref{eq. original objective}, for learning the expert team behavior model via distribution matching.

\citet{seo2024idil} suggest that matching \AMoccu occupancy measure amounts to matching two factored occupancy measures, $\OMpi{}{}$ and $\OMtx{}{}$, simultaneously with their expert counterparts in the single-agent problem. We refer to these factored occupancy measures as \PIoccu occupancy measure and \TXoccu occupancy measure, respectively, and define them for each agent $i$ as:
\begin{align*}
    \OMpi{i}{} &= \sumP{\TmO{i}{t}\myeq o_i, \TmA{i}{t}\myeq a_i, \TmX{i}{t}\myeq x_i|\JoAM, \calM}\\
    \OMtx{i}{} &= \sumP{\TmO{i}{t}\myeq o_i, \TmX{i}{t}\myeq x_i, \TmX{i}{t-1}\myeq \PrX{i}|\JoAM, \calM}
\end{align*}
With the factored occupancy measures, we can further factorize Eq. \ref{eq. original objective} as follows:
\begin{align}
&\argmin_{\JoAM} \sum_{i=1}^n \left(\Dfpi{i}{E}{i}{} \right.  \nonumber \\
&\hspace{20ex}\left.+ \Dftx{i}{E}{i}{}\right)  \label{eq. actual objective}
\end{align}
The proof for Eq. \ref{eq. actual objective}, along with the adjusted theorems that formulate this factored objective for the multi-agent scenario, is provided in the Appendix.

\section{Deep Team Imitation Learner}
With the theoretical foundations established in the previous section, we now present \ouralg, a practical algorithm designed to minimize Eq. \ref{eq. actual objective}. The distribution matching framework enables policies to be represented using deep neural networks and efficiently learns them by leveraging additional interactions with the environment. As a result, \ouralg is capable of learning team behavior models even in highly complex tasks.

As mentioned in Sec. \ref{sec. distribution matching}, the expert occupancy measure is typically estimated from expert demonstrations, i.e., $\OMam{}{} \approx \EX_{\DemoTeam}[\Ind(\oaxx{})]$. However, as our demonstration $\DemoTeam$ does not contain the labels of subtasks, we cannot compute this empirical distribution. Thus, similar to \cite{jing2021adversarial}, we take an expectation-maximization (EM) approach to iteratively optimize Eq. \ref{eq. actual objective}. Alg. \ref{alg. learner} outlines \ouralg. In line 4 (E-step), it predicts unknown expert intents from $\DemoTeam$ using the current estimate of agent models $(\PaPiTx{i}{k}{k})$. In line 5, it collects online samples by interacting with the environment. Then, in line 6 (M-step), it updates agent model parameters $(\PiParam, \TxParam)$ via occupancy measure matching.

\begin{algorithm}[t]
    \caption{\ouralg: \ouralgfull}
    \label{alg. learner}
    \begin{algorithmic}[1]
        \STATE \textbf{Input}: Data $\DemoTeam = \{\Traj{m}\}_{m=1}^d$ and $\XLabel{m}{m=1}{l}$.
        \STATE \textbf{Initialize}: $(\PiParam_i, \TxParam_i)$ for all $i=1:n$ where $\PaAM{i}{} = (\PaPiTx{i}{}{})$
        \REPEAT
            \STATE \textit{E-step}: Infer expert intents $\XLabel{m}{m=l}{d}$ with $\DemoTeam$ and $(\PaPiTx{}{k}{k})$ for all $i=1:n$; and define $\DemoTeamX\doteq \DemoTeam \cup \XEst{m}{m=1}{d}$
            \STATE Collect rollouts $\RolloutTeam = \{(\obs, x, a)^{0:h}\}$ using $(\PaPiTx{}{k}{k})$
            \STATE \textit{M-step}: Update $\PaPi{i}{k+1}$ via Eq. \ref{eq. pi objective} and  $\PaTx{i}{k+1}$ via Eq. \ref{eq. tx objective} using $\DemoX{i}, \Rollout{i}$ for all $i=1:n$
        \UNTIL{Convergence}
    \end{algorithmic}
\end{algorithm}

\paragraph{E-step.}
For each iteration $k$, \ouralg infers the unknown subtasks of demonstration $\Traj{}=(\TmOA{}{0:h})$ based on the maximum a posteriori (MAP) estimation. Given the current estimate of agent models $\PaAM{}{k} = (\PaPiTx{}{k}{k})$, we can express the MAP estimation as follows: $ \XEstSeq{} = \argmax_{\TmX{}{0:h}} p(\TmX{}{0:h}|\TmOA{}{0:h}, \PaAM{}{k}) $.
Similar to the Viterbi algorithm, this can be effectively computed via dynamic programming \cite{seo2022semi,jing2021adversarial}. Since its computation can be decentralized for each agent, its time complexity is $O(nh|\bar{X}|^2)$ where $|\bar{X}| \doteq \max_{i=1:n} |X_i|$.
The derivations are provided in the Appendix. With this estimate, we can obtain subtask-augmented trajectories $\TrajX{} = (o, a, \xhat)^{0:h}$. From $\DemoTeamX \doteq \DemoTeam \cup \Set{\XSeq{m}}{m=1}{d}=\Set{(o, a, \xhat)^{0:h}}{}{}$, we can compute the estimates of the expert occupancy measures for the $k$-th iteration, denoted by $\OMamEst{E}{}{k}$, $\OMpiEst{E}{}{k}$, and $\OMtxEst{E}{}{k}$, respectively.

\paragraph{M-step.}
We incrementally update the agent model parameters $(\PiParam, \TxParam)$ to minimize the difference between the learner's occupancy measure $\rho_{\calN}$ and the $k$-th estimate of expert occupancy measure $\rhoE^k$. Similar to IDIL \cite{seo2024idil}, \ouralg takes a factored approach to minimize Eq. \ref{eq. actual objective}. 
Specifically, when updating the low-level policy parameters $\PiParam_i$, it assumes $\Tx_i$ is fixed and minimizes only the first term of Eq. \ref{eq. actual objective} with respect to $\pi_{i}$: 
\begin{align}
    \argmin_{\pi_i} \Dfpi{\pi_i}{E}{i}{k} \label{eq. pi objective}
\end{align}
If we introduce $\Spi \doteq (\obs, x)$, this is the same as the occupancy-measure matching of a conventional policy $\pi(a|\Spi)\doteq \pi(a|\obs, x)$. 

Similarly, \ouralg updates the high-level policy parameters $\TxParam_i$ by only minimizing the second term of Eq. \ref{eq. actual objective} with respect to $\Tx_i$, fixing $\pi_i$:
\begin{align}
    \argmin_{\Tx_i} \Dftx{\Tx_i}{E}{i}{k}  \label{eq. tx objective}
\end{align}
This also reduces to conventional imitation learning of a policy $\pi(x|\Stx) \doteq \Tx(x|\obs,\PrX{})$, if we define $\Stx \doteq (\obs, \PrX{})$.
In this work, we opt for IQ-Learn \cite{garg2021iq} for both Eq. \ref{eq. pi objective} and Eq. \ref{eq. tx objective} to compute the gradients of $\PiParam$ and $\TxParam$, respectively.

\section{Convergence Properties}
\newcommand{\postPrX}[2]{p(\PrX{#1}|\oax{#1}, #2)}
\newcommand{\Dfpilemma}[1]{\Dfpi{\pi,\Tx}{E}{#1}{k}}
\newcommand{\Dfamlemma}[1]{\Dfamgrv{\pi,\Tx}{E}{#1}{k}}
While the optimization of Eq. \ref{eq. actual objective} will provide us with agent models whose occupancy measure is close to that of experts, it is not guaranteed that our practical, factored approach of iteratively minimizing each term of Eq. \ref{eq. actual objective} will converge. Although \idil provides theoretical analysis regarding the convergence of this factored distribution matching, their analysis is made under the assumption of full observability, thereby inapplicable to our setting. Thus, we provide a theoretical analysis regarding the convergence of \ouralg in this section.

\begin{figure*}[t]
\newcommand\gapa{0.175}
\newcommand\gapb{0.25}
\newcommand\subgap{1.}
  \begin{subfigure}[b]{\gapa\linewidth}
      \centering
      \includegraphics[height=35mm]{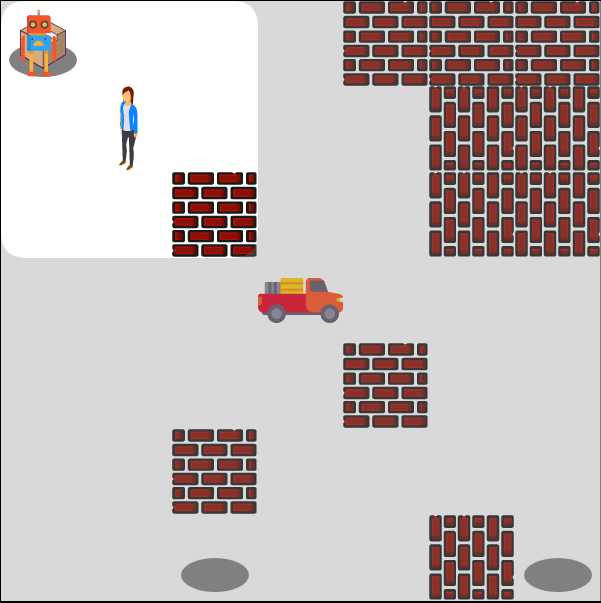}
      \caption{\movers}
      \label{fig: movers}
  \end{subfigure}
  \hfill 
  \begin{subfigure}[b]{\gapa\linewidth}
      \centering
      \includegraphics[height=35mm]{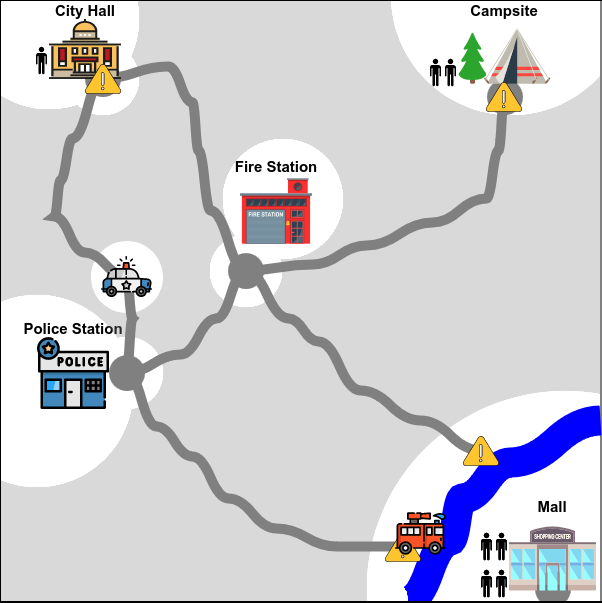}
      \caption{\rescue}
      \label{fig: flood}
  \end{subfigure}
  \hfill  
  \begin{subfigure}[b]{\gapb\linewidth}
      \centering
      \includegraphics[height=35mm]{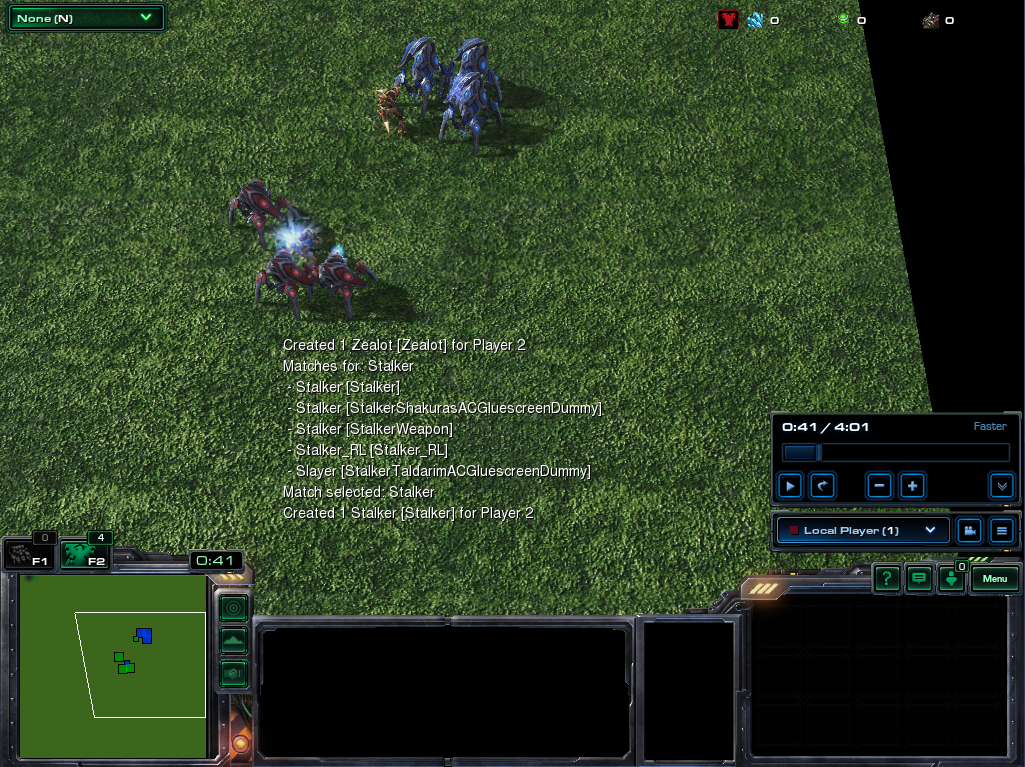}
      \caption{\protoss}
      \label{fig: protoss}
  \end{subfigure}
  \hfill  
  \begin{subfigure}[b]{\gapb\linewidth}
      \centering
      \includegraphics[height=35mm]{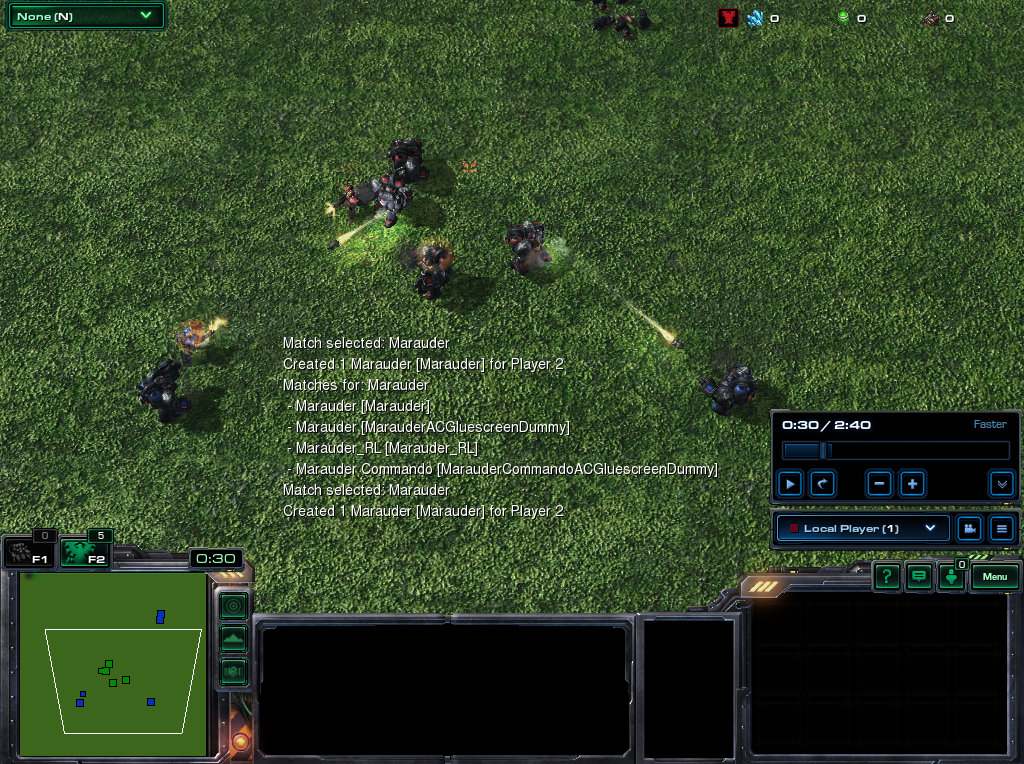}
      \caption{\terran}
      \label{fig: terran}
  \end{subfigure}

  \captionsetup{subrefformat=parens}
  \caption{Snapshopts of Experimental Domains}
  \label{fig: domains}
\end{figure*}

We start the analysis by defining two approximations of the expert \AMoccu occupancy measure, $\chkRho_E^k$ and $\grvRho_E^k$. These approximations are computed from the estimates of the expert's \PIoccu occupancy and \TXoccu occupancy measures, i.e., $\OMpiEst{E}{}{k}$ and $\OMtxEst{E}{}{k}$, with the estimate of expert models $\calN^{k} = (\pi^{k}, \Tx^{k})$:
\begin{align*}
    \chkRho_E^k(\oaxx{i}) &\doteq \OMtxEst{E}{i}{k}\pi_i^k(a_i|o_i, x_i) \\
    \grvRho_E^k(\oaxx{i}) &\doteq \OMpiEst{E}{i}{k} \postPrX{i}{\calN^k}.
\end{align*}
We can draw a relationship between the $\PIoccu$ occupancy measure matching and the $\AMoccu$ occupancy measure matching problems as follows:
\begin{lemma}
\label{thm. oax to oaxx}
    Define $\Delta(\PiParam, \PiParam^k)\doteq\epsilon_1$ and $\Delta(\TxParam, \TxParam^k)\doteq\epsilon_2$. If $\PaPi{}{}$ is an $K_1$-Lipschitz function of $\PiParam$, $\PaTx{}{}$ is an $K_2$-Lipschitz function of $\TxParam$, and $\max(K_1, K_2)(|\epsilon_1| + |\epsilon_2|)$ is sufficiently small, then 
    \begin{align*}
        &\Dfpilemma{i} \\
        &\hspace{20ex}=\Dfamlemma{i}
    \end{align*}
\end{lemma}
The proof of Lemma \ref{thm. oax to oaxx} is based on the first-order approximation of $f$ and provided in the Appendix. This implies that in reasonable conditions (e.g., smoothness of neural networks and compactness of parameter space), if we update $\PiParam$ only by a small amount via Eq. \ref{eq. pi objective}, our objective function, Eq. \ref{eq. original objective}, also decreases.

Then, along with \textit{Lemma 2.3} from \cite{seo2024idil}, we can derive the following theorem for the convergence of \ouralg.
\begin{theorem} 
\label{thm. convergence}
Let $\Loss{k}\doteq\Dfam{i}{E}{i}{k}$, and $p_E(o_i, a_i)$ denote the stationary distributions of $o, a$ computed from the expert demonstrations $\DemoTeam$. If (1) the conditions of lemma \ref{thm. oax to oaxx} is satisfied and (2) $\chkRho_E^k\approx\grvRho_E^k\approx p(x_i, \PrX{i}|o_i, a_i, \JoAM^k) p_E(o_i, a_i)$, then $\Loss{k+1}\leq \Loss{k}$.
\end{theorem}
The proof is built on the convexity of the $f$ and the minimization of $f$-divergence. Its details are provided in the Appendix. With Thm. \ref{thm. convergence}, since the objective function, $\Loss{}$, is always positive, it will eventually converge to a local optimum. Note that without any information regarding the rules or labels of the subtasks, multiple solutions can exist to exhibit the expert demonstrations $\DemoTeam$. Our approach allows for semi-supervision by incorporating expert subtask labels in the E-step. As the experimental results demonstrate, semi-supervision can help disambiguate the models, finding one closer to the actual expert team behavior.

\section{Experiments}

Through numerical experiments, we now assess \ouralg's performance against MAIL baselines across multiple domains.

\subsection{Experimental Setup}

\subsubsection{Domains}
We evaluate \ouralg across multiple domains with varying complexity, including the \simplemulti-$n$ suite, \movers, \rescue, and the \smactwo suite. 
The key characteristics of our experimental domains are presented in Table \ref{table. domains}. Our domains include both continuous and discrete observation and action spaces, with varying numbers of agents (2-5) who are either subtask-agnostic or subtask-driven.
Please refer to Figs.~\ref{fig: individual paths}--\ref{fig: workplace} for illustrations of \simplemulti-$n$ domains. Remaining domains are illustrated in Fig.~\ref{fig: domains}.

\begin{table}[t]
\caption{Key Characteristics of Experimental Domains. ``Subtask'' refers to whether agents are subtask-driven. ``Dim'' denotes the dimension or cardinality of a space.}
\newcommand{\mcb}[2]{\multicolumn{#1}{c}{\bf #2}}

\label{table. domains}
\begin{center}
\begin{tabular}{lccccccc}  \toprule 
                & \mcb{2}{Experts}   & \mcb{2}{\shortstack{Observation\\Space}} &\mcb{2}{\shortstack{Action\\Space}} \\ 
                \cmidrule(lr){2-3}       \cmidrule(lr){4-5}            \cmidrule(lr){6-7}
\textbf{Domain} & \# agents & Subtask  & Type        & $Dim$       & Type       & $Dim$  \\ \midrule
\simplemultish-2  & 2        & Yes                  & Cont.  & 6           & Cont. & 2                  \\
\simplemultish-3  & 2        & Yes  & Cont.  & 6           & Cont. & 2                 \\
\movers         & 2     & Yes  & Disc.    & 45          & Disc.   & 6                 \\
\rescue         & 2   & Yes  & Disc.    & 56          & Disc.   & 6                  \\
\protosssh        & 5   & No   & Cont.  & 90          & Disc.   & 11                  \\
\terransh         & 5   & No   & Cont.  & 80          & Disc.   & 11                \\
\bottomrule      
\end{tabular}
\end{center}
\end{table}

The \simplemulti-$n$ simulates the motivating example introduced in Fig. \ref{fig: workplace} in continuous observation and action spaces. \movers and \rescue are collaborative team tasks in partially observable environments introduced by \cite{seo2023automated}, considering only discrete states and actions. These domains are designed to admit multiple near-optimal strategies, allowing agents to exhibit multimodal behaviors. We create synthetic agents exhibiting hierarchical behavior and generate 50 and 100 demonstrations for training and testing, respectively, for each domain. \smactwo is a challenging benchmark for multi-agent reinforcement learning \cite{ellis2023smacv2}. We consider two domains in this suite: \protoss and \terran, where a team of five agents is tasked with defeating five enemies. We obtain a multi-agent policy via MAPPO \cite{yu2022the} and generate 50 trajectories per domain for training. In all domains, team members must make decentralized decisions in partially observable environments. For more details, please refer to the Appendix.

\paragraph{Baselines} 
We compare our approach with Behavior Cloning (BC), \magail (\magailsh) \cite{song2018multi}, \iiql (\iiqlsh), and \maogail (\maogailsh).  BC is a supervised learning approach to learning policies, which serves as a fundamental baseline for imitation learning \cite{li2022rethinking}. \magail is a generative adversarial training-based \mail algorithm, which employs the centralized training with decentralized execution (CTDE) approach
\cite{goodfellow2014generative}.
\iiql is a naive multi-agent extension of \iql \cite{garg2021iq}, which applies \iql to each agent independently. Since this baseline also takes non-adversarial training, we can compare the effect of hierarchical structure in modeling team behavior with this baseline. 
To our knowledge, no approach exists for learning hierarchical policies in complex multi-agent domains. Thus, we present \maogail as a baseline, which learns a hierarchical policy of each agent separately via \ogail \cite{jing2021adversarial}. For discrete domains, \movers and \rescue, we also report the performance of \btil \cite{seo2022semi}.

\begin{table}[t]
\caption{Average cumulative task reward with multimodal team behavior. \maogailsh-s and \ouralgsh-s represent the results with 20\% supervision.}
\label{table. task reward results}
\begin{center}
\newcommand{\mcb}[1]{\multicolumn{1}{c}{\bf #1}}
\newcommand{\MSb}[2]{\textbf{\MS{#1}{#2}}}
\newcommand{\mcbl}[1]{\multicolumn{1}{c:}{\bf #1}}
\newcommand{\mr}[1]{\multirow{2}{*}{#1}}
\setlength\tabcolsep{3pt} 
\begin{tabular}{ccccc} \toprule
\mcb{Method}  & \mcb{\simplemultish-$2$} & \mcb{\simplemultish-$3$} & \mcb{movers}   & \mcb{rescue}   \\ 
\midrule
\mcb{Expert}  & \MS{24.1}{3.8}           & \MS{28.7}{4.6}           & \MS{-99}{32}  & \MS{4.6}{2.0}  \\
\midrule
\mcb{BC}      & \MS{9.6}{3.2}            & \MS{11.8}{1.5}           & \MS{-150}{0}  & \MS{0.0}{0.0}  \\
\mcb{\magailsh}  & \MS{6.9}{2.0}         & \MS{10.4}{1.2}           & \MS{-150}{0}  & \MS{4.5}{0.5}  \\
\mcb{\iiqlsh}    & \MS{14.7}{0.7}        & \MSb{27.8}{1.6}          & \MSb{-107}{9} & \MSb{5.6}{0.2} \\
\mcb{\maogailsh} & \MS{6.1}{1.3}         & \MS{7.4}{2.0}            & \MS{-150}{0}  & \MS{3.6}{0.6}  \\
\mcb{\ouralgsh}  & \MSb{16.8}{5.9}       & \MS{27.0}{1.0}           & \MS{-108}{7} & \MS{5.4}{0.4} \\
\midrule
\mcb{\btil}        & -      & -           & \MS{-150}{0}  & \MS{0.6}{0.4}  \\
\mcb{\maogailsh-s} & \MS{11.8}{1.3}      & \MS{10.2}{2.0}           & \MS{-150}{0}  & \MS{3.8}{0.7}  \\
\mcb{\ouralgsh-s} & \MSb{21.6}{1.7}      & \MSb{27.3}{1.3}          & \MSb{-99}{14} & \MSb{4.9}{0.1} \\
\bottomrule
\end{tabular}
\end{center}
\end{table}

\begin{table}[t]
\caption{Average cumulative task reward and the rate of wins in \smactwo domains.}
\label{table. smacv2 results}
\begin{center}
\newcommand{\mcb}[1]{\multicolumn{1}{c}{\bf #1}}
\newcommand{\MSb}[2]{\textbf{\MS{#1}{#2}}}
\newcommand{\mcbl}[1]{\multicolumn{1}{c:}{\bf #1}}
\newcommand{\mr}[1]{\multirow{2}{*}{#1}}
\setlength\tabcolsep{3pt} 
\begin{tabular}{ccccc} \toprule
\multirow{2}{*}{\bf Method}  & \multicolumn{2}{c}{\bf \protosssh} & \multicolumn{2}{c}{\bf \terransh} \\ 
\cmidrule(lr){2-3} \cmidrule(lr){4-5}
 & \mcb{Reward} & \mcb{Wins} & \mcb{Reward} & \mcb{Wins} \\ 
\midrule
\mcb{Expert}   & \MS{18.0}{4.8}  & \MS{0.55}{0.50}  & \MS{11.9}{3.2}  & \MS{0.60}{0.49}  \\
\mcb{BC}       & \MS{6.5}{0.1}   & \MS{0.17}{0.06}  & \MS{4.7}{1.8}   & \MS{0.07}{0.06}  \\
\mcb{\magailsh}  & \MS{7.8}{1.0}  & \MS{0.27}{0.06}  & \MS{7.2}{1.3}   & \MS{0.27}{0.6}   \\
\mcb{\iiqlsh}    & \MS{8.4}{0.7}  & \MSb{0.47}{0.06} & \MS{7.8}{0.7}   & \MSb{0.47}{0.06} \\
\mcb{\maogailsh} & \MS{6.8}{1.0}  & \MS{0.13}{0.06}  & \MS{6.2}{1.3}   & \MS{0.23}{0.06}  \\
\mcb{\ouralgsh}  & \MSb{9.9}{1.2} & \MSb{0.47}{0.06} & \MSb{8.2}{1.3}  & \MSb{0.47}{0.06} \\
\bottomrule
\end{tabular}
\end{center}
\end{table}

\paragraph{Metrics} 
Similar to other imitation learning algorithms \cite{song2018multi, garg2021iq}, we use the cumulative task reward to evaluate the algorithms. In \smactwo domains, we also consider the win rate in battles. If the learned multi-agent policy aligns with the expert team behavior, it will achieve a task reward similar to the expert's. However, a high task reward does not necessarily indicate alignment with the expert team model, as the algorithms might learn only one optimal policy, resulting in unimodal rather than multimodal behavior. Therefore, we also measure the accuracy of subtask inference. The closer the learned model is to the expert model, the more accurately it can predict expert subtasks from their demonstrations. We use the MAP estimation (the E-step of Alg. \ref{alg. learner}) for subtask inference.

\subsection{Results}
\label{sec. exp results}

\subsubsection{\ouralgsh achieves expert-level task performance.} 
Table \ref{table. task reward results} shows the task rewards averaged over three trials for \simplemulti-$n$ (\simplemultish-$n$), \movers, and \rescue. We observe that \iql-based approaches, \iiqlsh and \ouralgsh, generally perform better than approaches based on generative adversarial imitation learning. Between \magailsh and \maogailsh, \magailsh performed better, likely because \magailsh additionally utilizes other agents' information during training. In contrast, \maogailsh and \ouralgsh take only individual observation-action trajectories and do not utilize any other information that might break the partial observability condition even during the training phase.

While \ouralgsh outperformed \iiqlsh in \simplemulti-$2$, it performed on par in other domains. We believe this is due to the domains being too simple, allowing subtask-agnostic approaches to extract one optimal solution from demonstrations. In more complex domains, we could observe an improvement in task performance due to the hierarchical structure of our agent model. As shown in Table \ref{table. smacv2 results}, \ouralgsh achieved the highest task reward and win rate in both the \smactwo domains: \protoss and \terran. We want to highlight that even though \iiqlsh often achieves high task rewards, it can neither learn multimodal behavior nor utilize semi-supervision. On the other hand, \ouralgsh can improve its performance by augmenting subtask labels. As shown in Table \ref{table. task reward results}, \ouralgsh achieved a task reward similar to the expert task reward in all domains with 20\% supervision.

\begin{table}[t]
\caption{Accuracy of Subtask Inference. We represents the $1$-st agent of the \simplemultish-$2$ domain as \simplemultish-$2$-1, and similarly for other agents.}%
\label{table. inference}
\newcommand{\mcb}[1]{\multicolumn{1}{c}{\bf #1}}
\newcommand{\MSb}[2]{\textbf{\MS{#1}{#2}}}
\newcommand{\mcbl}[1]{\multicolumn{1}{c:}{\bf #1}}
\newcommand{\mr}[1]{\multirow{2}{*}{#1}}
\begin{tabular}{l@{}c@{}c@{}c@{}c} \toprule
\mcb{Agent} & \mcb{Random} & \mcb{\btil} & \mcb{\maogailsh-s} & \mcb{\ouralgsh-s}\\
\midrule
\simplemultish-$2$-1 & $\approx0.50$ & - & \MS{0.61}{0.09} & \MSb{0.75}{0.04}  \\
\simplemultish-$2$-2 & $\approx0.50$ & - & \MS{0.63}{0.17} & \MSb{0.75}{0.07}  \\
\simplemultish-$3$-1 & $\approx0.33$ & - & \MS{0.49}{0.04}   & \MSb{0.78}{0.07}  \\
\simplemultish-$3$-2 & $\approx0.33$ & - & \MS{0.68}{0.06}   & \MSb{0.72}{0.08}  \\
\movers-1 & $\approx0.25$ & \MSb{0.90}{0.01} & \MS{0.35}{0.15} & \MS{0.78}{0.02} \\
\movers-2 & $\approx0.25$ & \MSb{0.91}{0.01} & \MS{0.46}{0.07} & \MS{0.78}{0.07} \\
\rescue-1 & $\approx0.25$ & \MS{0.53}{0.04} & \MS{0.31}{0.06}  & \MSb{0.61}{0.08} \\
\rescue-2 & $\approx0.25$ & \MSb{0.62}{0.01} & \MS{0.25}{0.12}  & \MS{0.57}{0.03} \\
\bottomrule
\end{tabular}
\end{table}

\begin{figure*}[t]
  \def\subfwid{.45}
  \def\figwid{.3}
  \centering
  \begin{subfigure}[t]{\subfwid\linewidth}
      \centering
      \includegraphics[width=\figwid\linewidth, frame]{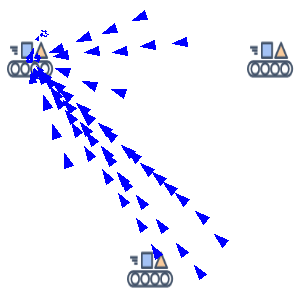}\hspace{0.5ex}
      \includegraphics[width=\figwid\linewidth, frame]{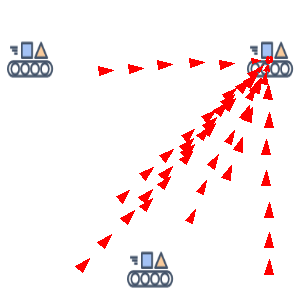}\hspace{0.5ex}
      \includegraphics[width=\figwid\linewidth, frame]{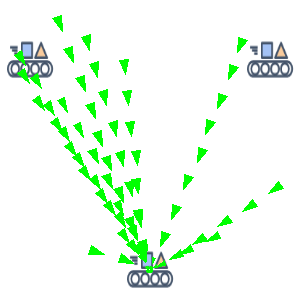}
      \caption{Expert (Agent 1)}
      \label{fig: expert a1 visualization}
  \end{subfigure}
  \hspace{1ex}
  \begin{subfigure}[t]{\subfwid\linewidth}
      \centering
      \includegraphics[width=\figwid\linewidth, frame]{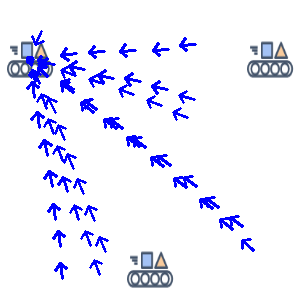}\hspace{0.5ex}
      \includegraphics[width=\figwid\linewidth, frame]{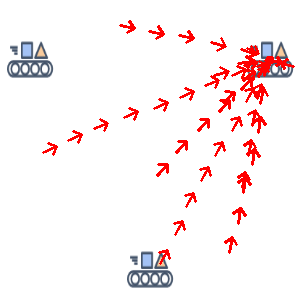}\hspace{0.5ex}
      \includegraphics[width=\figwid\linewidth, frame]{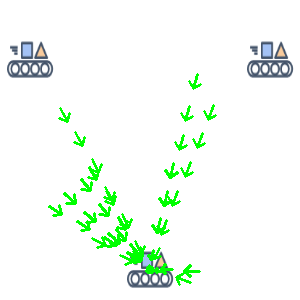}
      \caption{Expert (Agent 2)}
  \end{subfigure}
  \par\bigskip 
  \begin{subfigure}[t]{\subfwid\linewidth}
      \centering
      \includegraphics[width=\figwid\linewidth, frame]{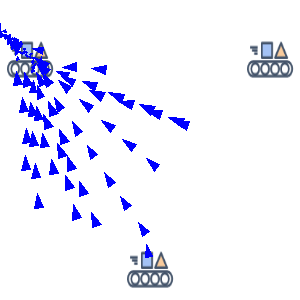}\hspace{0.5ex}
      \includegraphics[width=\figwid\linewidth, frame]{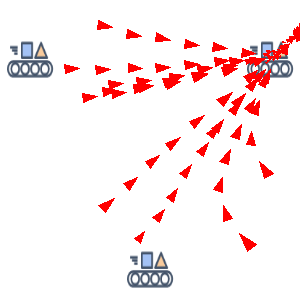}\hspace{0.5ex}
      \includegraphics[width=\figwid\linewidth, frame]{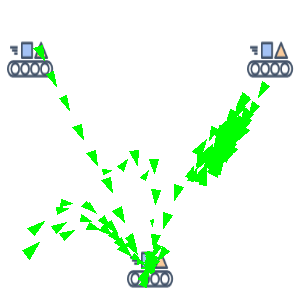}
      \caption{\ouralg (Agent 1)}
  \end{subfigure}
  \hspace{1ex}
  \begin{subfigure}[t]{\subfwid\linewidth}
      \centering
      \includegraphics[width=\figwid\linewidth, frame]{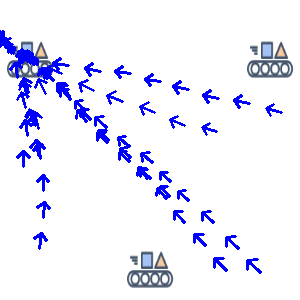}\hspace{0.5ex}
      \includegraphics[width=\figwid\linewidth, frame]{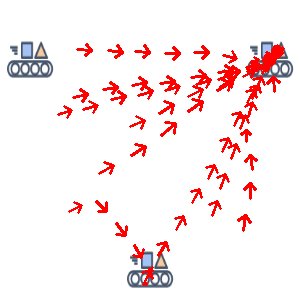}\hspace{0.5ex}
      \includegraphics[width=\figwid\linewidth, frame]{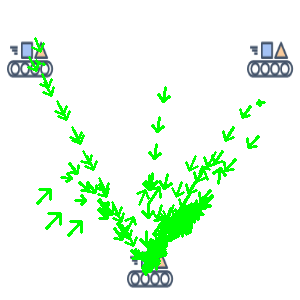}
      \caption{\ouralg (Agent 2)}
  \end{subfigure}
  \par\bigskip 
  \begin{subfigure}[t]{\subfwid\linewidth}
      \centering
      \includegraphics[width=\figwid\linewidth, frame]{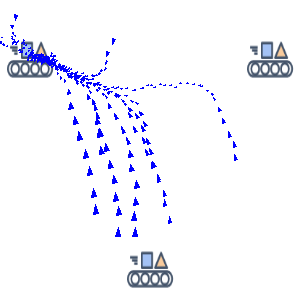}\hspace{0.5ex}
      \includegraphics[width=\figwid\linewidth, frame]{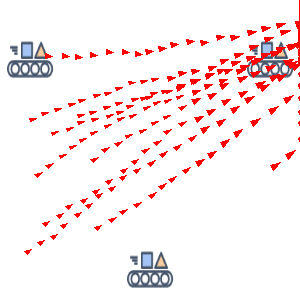}\hspace{0.5ex}
      \includegraphics[width=\figwid\linewidth, frame]{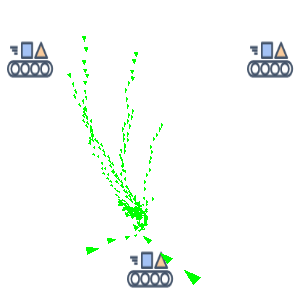}
      \caption{\maogail (Agent 1)}
  \end{subfigure}
  \hspace{1ex}
  \begin{subfigure}[t]{\subfwid\linewidth}
      \centering
      \includegraphics[width=\figwid\linewidth, frame]{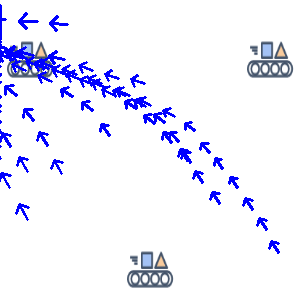}\hspace{0.5ex}
      \includegraphics[width=\figwid\linewidth, frame]{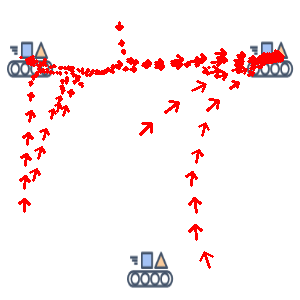}\hspace{0.5ex}
      \includegraphics[width=\figwid\linewidth, frame]{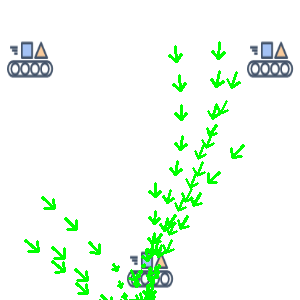}
      \caption{\maogail (Agent 2)}
  \end{subfigure}
  \captionsetup{subrefformat=parens}
  \caption{Visualization of individual \simplemulti-$3$ trajectories generated by the expert and learned models conditioned on a fixed subtask. The directions of the triangles and arrows represent the actions of agents at each position. The three colors represent the three fixed subtasks. Both learned models (\ouralg and \maogail) are trained with 20 \% supervision of subtask labels.}
  \label{fig: individual paths}
\end{figure*}

\subsubsection{\ouralgsh accurately learns multimodal team behavior.} As mentioned in Sec. \ref{sec. problem formulation}, our goal is to learn the different team behaviors generated by expert teams rather than a unimodal team policy. Additionally, in human-AI teaming applications, an AI agent must accurately interpret its human teammate's high-level plan. To achieve this, it is essential to learn a model that exhibits hierarchical behavior aligned with expert team members.
Table. \ref{table. inference} presents the accuracy of subtask inference computed with 20\%-supervision models. Note that without any supervision, we cannot associate the learned latent values with the actual subtasks. In all cases, \ouralgsh outperforms \maogailsh and the random baselines.\footnote{Note that other DNN-based baselines cannot be utilized for subtask inference.} 

\subsubsection{\ouralgsh outperforms \btil in more complex tasks.}
As demonstrated in Table \ref{table. task reward results}, \ouralg outperformed \btil in terms of task reward. While \btil's overall performance was generally below that of online methods such as \maogail and \ouralg, it performed slightly better than BC. We attribute this to \btil's offline nature, which makes it prone to compounding errors. This implies that if a \btil agent encounters a state that was not present in the training dataset, it struggles to select the appropriate action, as it has not learned anything about that state.
On the other hand, \btil's subtask inference performance was on par with, or even superior to, \ouralg. As shown in Table \ref{table. inference}, \btil achieved approximately 0.9 accuracy in subtask inference for \movers. However, we emphasize that this level of performance is only feasible in small domains, as \btil cannot scale up to domains with larger state spaces.

\subsubsection{Team models learned using DTIL generated behaviors that are qualitatively similar to those generated by expert teams.}
To interpret the learned behavior associated with each subtask ($x$), we visualized the paths generated by ten different seeds (seed=0:9) for each model in Figure \ref{fig: individual paths}. 
For this visualization, we intentionally set the part of an agent’s observation related to other agents to zero, eliminating their influence on the behavior of the agent being inspected. 
Figure \ref{fig: individual paths} shows that \ouralg's subtask-driven behavior closely resembles the expert's when trained with partial supervision of subtask labels. In contrast, \maogail trajectories tend to be noisy and unfocused on a specific subgoal, even with a fixed $x$. 
We believe the superior performance of \ouralg stems from its factored structure, in which it learns separate Q-functions for $\pi$ and $\Tx$. Given that Q-functions can be interpreted as reward functions \cite{garg2021iq}, our approach effectively learns a hierarchical reward corresponding to each level of the policy. However, since \maogail does not learn separate Q-functions, it is difficult to determine whether it truly captures a hierarchical policy structure or merely optimizes the joint policy $\pi(a, x|o, x^-)$.

\section{Conclusion}
This work introduces \ouralg, an algorithm for learning generative models of team behavior from heterogeneous demonstrations. Experiments show that \ouralg outperforms state-of-the-art multi-agent imitation learning baselines and captures expert team behavior across six diverse teamwork domains. Additionally, \ouralg can generate a wide range of expert team behaviors.
\ouralg also motivates future research directions. First, \ouralg assumes a known, finite set of subtasks, though real-world subtasks may be difficult to define a prior or represent as scalars. Future MAIL methods should explore more expressive hierarchical representations. Second, by enabling generative models of team behavior, \ouralg can enable novel human-AI teaming applications, such as AI-enabled team coaching~\cite{seo2021towards, seo2025socratic} and end-user programming of multi-agent systems~\cite{ajaykumar2021survey, stegner2024understanding}. We invite developers of these and other impactful applications to utilize \ouralg and make additional details available at \url{http://tiny.cc/dtil-appendix}



\begin{acks}
This research was supported in part by NSF award $\#2205454$, ARO CA\# W911NF-20-2-0214, and Rice University funds.
\end{acks}


\ifarxiv
\appendix
\onecolumn
\section{Proofs and Derivations}
\subsection{Theorem \ref{thm. bijection}}
\newcommand{\VWoccu}{\textsf{\Sam\Aam}}
\newcommand{\OMfull}{\rho(\obs, s, a, x, \PrX{}) }
\newcommand{\pimerged}{\tilde{\pi}_i(a, x|\obs, \PrX{})}
\newcommand{\PiParamk}{\PiParam^k}
\newcommand{\TxParamk}{\TxParam^k}
\begin{proof}
Given a task model $\calM$ and joint agent models $\calN_{1:n}$, we define a stationary distribution over the joint variables $(\obs, s, a, x, \PrX{})$ as 
\begin{align*}
    \OMfull = \sumP{\TmO{}{t}\myeq\obs, \TmS{t}\myeq s, \TmA{}{t}\myeq a, \TmX{}{t}\myeq x, \TmX{}{t-1}\myeq\PrX{}|\calM, \JoAM}.
\end{align*}
Let $\Sam \doteq(\obs, \PrX{})$ and $\Aam \doteq(x, a)$ denote the subtask-augmented state and action, respectively. From $\calM$ and $\calN_{1:n}$, we define a stationary policy $\tilde{\pi}_i$, a stationary distribution of state transition $\tilde{T}_i$, and an initial distribution $\tilde{\mu}_i$ for each agent $i$ as follows:
\begin{align*}
    \tilde{\pi}_i(\Aam_i|\Sam_i) &\doteq \pi_i(a_i|\obs_i, x_i)\Tx_i(x_i|\obs_i, \PrX{i}) \\
    \tilde{T}_i(\Sam'_i|\Sam_i, \Aam_i)&\doteq\sum_{s, s', \OtA{i}}O_i(\obs_i'|s', a)T(s'|s, a)p(s, \OtA{i}|\obs_i, \PrX{i}, x_i, a_i) \\
    \tilde{\mu}_{0, i}(\Sam_i) &\doteq \sum_s \mu_0(s)O_i(\obs_i|s, \PrA{}\myeq\#)
\end{align*}
where $p(s, \OtA{i}|\obs_i, \PrX{i}, x_i, a_i) = \frac{\sum_{\Others{\obs}{i},\Others{x}{i}, \PrX{\Others{}{i}}} \OMfull}{\sum_{\Others{\obs}{i},\Others{x}{i}, \PrX{\Others{}{i}}, s, \Others{a}{i}} \OMfull}$. 
With $\tilde{\pi}_i,\tilde{T}_i$ and $\tilde{\mu}_{0, i}$, we can drive the following $\tilde{\pi}_i$-specific Bellman flow constraints \cite{puterman1990markov}:
\begin{align*}
    \rho_{\tilde{\pi}_i}(\Sam, \Aam) &= \sum_{t=0}^\infty \gamma^t p(\Sam^t\myeq \Sam, \Aam^t\myeq \Aam) \\
    &= p(\Sam^0\myeq \Sam, \Aam^0\myeq \Aam) + \gamma\sum_{t=0}^\infty \gamma^t \sum_{\Sam', \Aam'}p(\Sam^{t+1}\myeq \Sam, \Aam^{t+1}\myeq \Aam, \Sam^t\myeq \Sam', \Aam^t\myeq \Aam') \\
    &= \tilde{\pi}_i( \Aam|\Sam) \tilde{\mu}_{0,i} (\Sam) + \gamma\sum_{\Sam', \Aam'}\sum_{t=0}^\infty \gamma^t \tilde{\pi}_i(\Aam|\Sam) \tilde{T}_i(\Sam|\Sam', \Aam') p(\Sam^t\myeq \Sam', \Aam^t\myeq \Aam') \\
    &= \tilde{\pi}_i( \Aam|\Sam) \tilde{\mu}_{0,i} (\Sam) + \gamma\tilde{\pi}_i(\Aam|\Sam) \sum_{\Sam', \Aam'} \tilde{T}_i(\Sam|\Sam', \Aam') \sum_{t=0}^\infty \gamma^tp(\Sam^t\myeq \Sam', \Aam^t\myeq \Aam') \\
    &= \tilde{\pi}_i( \Aam|\Sam) \left(\tilde{\mu}_{0,i} (\Sam) + \gamma \sum_{\Sam', \Aam'} \tilde{T}_i(\Sam|\Sam', \Aam') \rho_{\tilde{\pi}_i}(\Sam', \Aam') \right)\\
    \rho_{\tilde{\pi}_i}(\Sam, \Aam) &\geq 0.
\end{align*}
where the subscript $i$ is omitted from $\Sam$ and $\Aam$ for notational convienence. Then, the one-to-one correspondence between $\rho_{\tilde{\pi}_i}$ and $\tilde{\pi}_i$ can be drawn by the following theorem:
\begin{theorem}[Theorem 2 of \citet{syed2008apprenticeship}]
    Let $\rho$ satisfy the Bellman flow constraints:
    \begin{align*}
        \sum_a \rho_{sa} = \mu_s + \gamma \sum_{s',a'}\rho_{s'a'}T_{s'a's} \qquad \text{and}\qquad \rho_{sa} &\geq 0,
    \end{align*}
    and let $\pi_{sa} = \frac{\rho_{sa}}{\sum_a\rho_{sa}}$ be a stationary policy. Then, $\rho$ is the occupancy measure for $\pi$. Conversely, if $\pi$ is a stationary policy such that $\rho$ is its occupancy measure, then $\pi_{sa} = \frac{\rho_{sa}}{\sum_a \rho_{sa}}$ and $\rho$ satisfies the Bellman flow constraints.
\end{theorem}
Additionally, the following lemma establishes a bijection between $\tilde{\pi}(\Sam|\Aam)$ and $\calN_i=(\pi_i, \Tx_i)$:
\begin{lemma}[Lemma 3 of \citet{jing2021adversarial}]
    There is a bijection between $\pimerged$ and $(\pi_i(a|\obs, x), \Tx_i(x|\obs, \PrX{}))$, where $\pimerged =\pi_i(a|o, x)\Tx_i(x|o, \PrX{})$ and 
    \begin{align*}
    \pi_i(a|\obs, x)=\left.\frac{\pimerged}{\sum\limits_a\pimerged}\right|_{\forall \PrX{}}=\frac{\sum\limits_{\PrX{}}\pimerged}{\sum\limits_{a, \PrX{}}\pimerged}, \qquad \Tx_i(x|s, \PrX{})=\sum_a\pimerged.
    \end{align*}
\end{lemma}

Thus, since $\rho_i = \rho_{\calN_i, \OtAM{i}}=\rho_{\tilde{\pi}_i}$, a one-to-one correspondence exists between $\calN_i$ and $\rho_i$, where 
\begin{align*}
    \pi_i(a|\obs, x) = \frac{\sum_{\PrX{}} \OMam{i}{} }{\sum_{\PrX{},a} \OMam{i}{}}, \qquad \Tx_i(x|\obs, \PrX{}) = \frac{\sum_{a}\OMam{i}{}}{\sum_{a,x}\OMam{i}{}}.
\end{align*}

\end{proof}

\subsection{Equation \ref{eq. actual objective}: factored distribution matching objective}
We first extends \textit{Proposition 2.1} of \cite{seo2024idil} to the multi-agent, partially-observable setting:
\begin{lemma}
\label{thm. factorization} Given a multi-agent task model $\calM$, let $\JoAM=(\pi, \Tx)$ and $\JoAM'=(\pi', \Tx')$ be two joint agent models. Then, $\OMpi{\calN_i}{} = \OMpi{\calN_i'}{}$ and $\OMtx{\calN_i}{} = \OMtx{\calN_i'}{}$ for all $i$ if and only if $\OMam{\calN_i}{}=\OMam{\calN_i'}{}$ for all $i$.
\end{lemma}
\begin{proof}
(\textit{if Direction}) For each agent $i$, since $\rho_{\calN_i}(\oaxx{})=\rho_{\calN_i'}(\oaxx{})$, we can derive $\rho_{\calN_i}(\oax{})=\rho_{\calN_i'}(\oax{})$ and $\rho_{\calN_i}(\oxx{})=\rho_{\calN_i'}(\oxx{})$ by marginalizing out $\PrX{i}$ and $a_i$ from both sides, respectively. 

(\textit{if-only Direction}) From $\rho_{\calN_i}(\oxx{})=\rho_{\calN'_i}(\oxx{})$, we can derive $\rho_{\calN_i}(o, \PrX{})=\rho_{\calN'_i}(o, \PrX{})$ by marginalizing out $x$ on each side. Then, from the definition of the occupancy measure, $\rho_{\calN_i}(\oxx{})=\Tx_i(x|o, \PrX{})\rho_{\calN_i}(o, \PrX{})$ and $\rho_{\calN_i'}(\oxx{})=\Tx_i'(x|o, \PrX{})\rho_{\calN_i'}(o, \PrX{})$. Since $\rho_{\calN_i}(o, \PrX{})=\rho_{\calN'_i}(o, \PrX{})$ and $\rho_{\calN_i}(\oxx{})=\rho_{\calN'_i}(\oxx{})$, we can derieve $\Tx(\cdot) = \Tx'(\cdot)$. Similarly, we can drive $\rho_{\calN_i}(o, x)=\rho_{\calN'_i}(o, x)$ from $\rho_{\calN_i}(\oax{})=\rho_{\calN'_i}(\oax{})$ by marginalizing out $a$ on each side. Then, from $\rho_{\calN}(\oax{})=\pi(a|o, x)\rho_{\calN}(o, \PrX{})=\pi'(a|o, x)\rho_{\calN'}(o, x)=\rho_{\calN_i'}(\oax{})$, we can also derive $\pi(\cdot) = \pi'(\cdot)$. Since $\pi_i=\pi_i'$ and $\Tx_i = \Tx_i'$ for all agents $i=1:n$, by the one-to-one correspondence shown in Thm. \ref{thm. bijection}, $\rho_\calN(\oaxx{i}) = \rho_\calN(\oaxx{i})$ for all $i$.
\end{proof}

With Lemma \ref{thm. factorization}, we can prove the derivation of Eq. \ref{eq. actual objective}.
\begin{proof}
Since $f$-divergence is always positive and becomes zero when the two measures are the same, at the minimum, Eq. \ref{eq. original objective} will lead to $\rho_i(\oaxx{i})=\rho_E(\oaxx{i})$ for all $i$. Likewise, Eq. \ref{eq. actual objective} will make $\rho_i(\oax{i})=\rho_E(\oax{i})$ and $\rho_i(\oxx{i})=\rho_E(\oxx{i})$ for all $i$ at the minimum. Then, by the lemma \ref{thm. factorization}, Eq. \ref{eq. actual objective} also results in $\rho_i(\oaxx{i})=\rho_E(\oaxx{i})$ for all $i$. Due to the one-to-one correspondence between $\rho_{E}$ and $\calN_E$ derived from Thm. \ref{thm. bijection}, the solutions of both Eqs. \ref{eq. original objective} and \ref{eq. actual objective} become identical, $\calN_E=(\pi_E, \Tx_E)$.
\end{proof}

\subsection{Lemma \ref{thm. oax to oaxx}}

\begin{proof}
\newcommand{\rhoPikTx}[1]{\rho_{\pi^{#1}, \Tx^k}}
\newcommand{\rhoPiTxk}[1]{\rho_{\pi, \Tx^{#1}}}
\newcommand{\rhoPiTxNon}{\rho_{\pi,\Tx}}
\newcommand{\qPrX}[1]{p\left({\PrX{}}|\cdot, #1\right)}
\newcommand{\rhoEoaxx}[1]{\rho_E^k(#1)\qPrX{\pi^k, \Tx^k}}
\newcommand{\deltapi}{\delta_1}
\newcommand{\deltatx}{\delta_2}
    We can expand $\Dfpilemma{}$ as follows\footnote{Except where it is unclear, we omit the subscript $i$ from functions and variables here. Also, we often use ``$\cdot$'' instead of ``$\oax{}$'' or ``$\oaxx{}$'' for brevity.}:
    \begin{align}
        &\Dfpilemma{} = \sum_{\oax{}}\rho_E^k(\oax{}) f\left(\frac{\rhoPiTxk{}(\oax{})}{\rho_E^k(\oax{})} \right) \nonumber \\
        &= \sum_{\oaxx{}}\grvRho_E^k(\oaxx{}) f\left(\frac{\rhoPiTxk{}(\oax{})\qPrX{\pi^k, \Tx^k}}{\rhoEoaxx{\oax{}}} \right) \nonumber\\
        &= \sum_{\oaxx{}}\grvRho_E^k(\cdot) \left\{
        f\left(\frac{\rhoPiTxk{}(\cdot)\qPrX{\pi^k, \Tx^k}}{\rhoEoaxx{\cdot}}\right)
        - f\left(\frac{\rhoPiTxk{}(\cdot)\qPrX{\pi, \Tx}}{\rhoEoaxx{\cdot}}\right)
        \right. \nonumber\\
        &\hspace{17ex}\left.
        + f\left(\frac{\rhoPiTxk{}(\cdot)\qPrX{\pi, \Tx}}{\rhoEoaxx{\cdot}}\right)
        \right\}. \label{eq. df expansion}
    \end{align}
    Since the last term in Eq. \ref{eq. df expansion} is $\Dfamlemma{}$, we have the following equality:
    \begin{align}
        \Dfpi{i}{E}{}{k} =A + \Dfamlemma{}, \label{eq. df relationship}
    \end{align}
    where
    \begin{align}
        A &\doteq \sum_{\oaxx{}}\grvRho_E^k(\cdot) \left\{f\left(\frac{\rhoPiTxk{}(\cdot)\qPrX{\pi^k, \Tx^{k}}}{\rhoEoaxx{\cdot}}\right)
        - f\left(\frac{\rhoPiTxk{}(\cdot)\qPrX{\pi, \Tx}}{\rhoEoaxx{\cdot}}\right)\right\}.
    \end{align}
    Since $\pi_\PiParam$ is $K_1$-Lipschitz and $\Tx_\TxParam$ is $K_2$-Lipschitz, the following holds:
    \begin{align}
        &|\Delta (\pi, \pi^k)| = |\pi_\PiParam - \pi_{\PiParam^k}| \leq K_1 |\Delta(\PiParam, \PiParam^k)| = K_1 |\epsilon_1| \label{eq. pi lipschitz}\\
        &|\Delta (\Tx, \Tx^k)| = |\Tx_\TxParam - \Tx_{\TxParam^k}| \leq K_2 |\Delta(\TxParam, \TxParam^k)| = K_2 |\epsilon_2| \label{eq. tx lipschitz}
    \end{align}
    Let us define $\tilde{\pi}(a, x|o, \PrX{}) = \pi(a|o, x) \Tx(x|o, \PrX{})$. Then, we can derive the following relationships:
    \begin{align*}
        \Delta (\tilde{\pi}, \tilde{\pi}^k) &= |\tilde{\pi} -\tilde{\pi}^k| = |\pi\cdot \Tx - \pi^k \cdot \Tx^k| \\
        &= |\pi\cdot \Tx - \pi^k \cdot \Tx + \pi^k \cdot \Tx - \pi^k \cdot \Tx^k| \\
        &\leq |\pi\cdot \Tx - \pi^k \cdot \Tx| + |\pi^k \cdot \Tx - \pi^k \cdot \Tx^k| \\
        &\leq |\pi - \pi^k| + |\Tx - \Tx^k|  \qquad (\because \max |\pi| = 1 \text{ and } \max|\Tx|=1)\\
        &\leq K_1 |\Delta(\PiParam, \PiParamk)| + K_2|\Delta(\TxParam, \TxParamk)| \\
        &\leq \max(K_1, K_2)(|\epsilon_1| + |\epsilon_2|)
    \end{align*}
    Then, since $\delta \doteq |\Delta (\tilde{\pi}, \tilde{\pi}^k)|$ is sufficiently small by the lemma condition, we can derive the first-order approximation of $f$ at $\tilde{\pi}^k$. Then, the term $A$ can be expressed as follows:
    \begin{align}
        &A \approx  \sum_{\oaxx{}}\grvRho_E^k(\cdot) \left\{f'\left(\frac{\rhoPikTx{k}(\cdot)}{\rho_E^k(\cdot)}\right) \frac{1}{\rho_E^k(\cdot)} \frac{d}{d\tilde{\pi}}\rho_{\tilde{\pi}}(\cdot)\delta \right. \nonumber \\
        &\hspace{10ex}\left. - f'\left(\frac{\rhoPikTx{k}(\cdot)}{\rho_E^k(\cdot)}\right) 
        \frac{\qPrX{\pi^k, \Tx^k} \frac{d}{d\tilde{\pi}}\rho_{\tilde{\pi}}(\cdot) + \rhoPikTx{k}(\cdot) \frac{d}{d\tilde{\pi}}\qPrX{\tilde{\pi}}}{\rhoEoaxx{\cdot}} \delta
        \right\} \nonumber \\
        &= \sum_{\oaxx{}}
        -f'\left(\frac{\rhoPikTx{k}(\cdot)}{\rho_E^k(\cdot)}\right) \rhoPikTx{k}(\cdot) \frac{d}{d\tilde{\pi}}\qPrX{\tilde{\pi}} \delta\nonumber \\
        &=\frac{d}{d\tilde{\pi}}\sum_{\oaxx{}}
        -f'\left(\frac{\rhoPikTx{k}(\cdot)}{\rho_E^k(\cdot)}\right) \rhoPikTx{k}(\cdot)\qPrX{\tilde{\pi}} \delta \nonumber \\
        &=\frac{d}{d\tilde{\pi}}\sum_{\oaxx{}}
        -f'\left(\frac{\rhoPikTx{k}(\oax{})}{\rho_E^k(\oax{})}\right) \rhoPikTx{k}(\oax{})\postPrX{}{\tilde{\pi}} \delta \nonumber \\
        &=\frac{d}{d\tilde{\pi}}\sum_{\oax{}}
        -f'\left(\frac{\rhoPikTx{k}(\oax{})}{\rho_E^k(\oax{})}\right) \rhoPikTx{k}(\oax{})\delta \nonumber\\
        &=0 \label{eq. term A}
    \end{align}
    The last equality holds because the summation term does not depend on $\tilde{\pi}$. Since $A=0$ by Eq. \ref{eq. term A}, we can derive $\Dfpilemma{} = \Dfamlemma{}$ from Eq. \ref{eq. df relationship}.
\end{proof}

\subsection{Theorem \ref{thm. convergence}}
Before proving Thm. \ref{thm. convergence},  we first introduce the following relationship between the $\TXoccu$ occupancy measure matching and $\AMoccu$ occupancy measure matching problems:
\begin{lemma}[Lemma 2.3 of \citet{seo2024idil}]
\label{thm. oxx to oaxx}
    Define $|\Delta(\PiParam, \PiParam^k)|=\epsilon$. If $\epsilon$ is sufficiently small, then $\Dftx{\pi, \Tx}{E}{i}{k}=\Dfamchk{\pi, \Tx}{E}{i}{k}$.
\end{lemma}
Since this lemma applies to multi-agent partially observable settings without modification, we refer to \cite{seo2024idil} for the proof. This lemma implies that under the given condition, the minimization of Eq. \ref{eq. tx objective} also minimizes the difference of the \AMoccu occupancy measures between the learner and the expert.

With the lemmas \ref{thm. oax to oaxx} and \ref{thm. oxx to oaxx}, we can derive a sufficient condition for decreasing our objective, Eq. \ref{eq. actual objective}.
\begin{corollary}
    If (1) the conditions of lemma \ref{thm. oax to oaxx} is satisfied and (2) $\chkRho_E^k\approx\grvRho_E^k$, minimizing Eqs. \ref{eq. pi objective} and \ref{eq. tx objective} decreases Eq. \ref{eq. actual objective}.
\end{corollary}
\begin{proof}
    Under the given conditions, we can rewrite Eq. \ref{eq. actual objective} as follows:
    \begin{align*}
        &\Dfpi{i}{E}{i}{k} +  \Dftx{i}{E}{i}{k} \\
        &= 2\Dfamlemma{} \hspace{10ex} (\because \grvRho_E^k \approx \chkRho_E^k)
    \end{align*}
    Thus, the $\PiParam$-update via Eq. \ref{eq. pi objective} and the $\TxParam$-update via Eq. \ref{eq. tx objective} both decrease $\Loss{k}_i$ by the lemmas \ref{thm. oax to oaxx} and \ref{thm. oxx to oaxx}, respectively.
\end{proof}

\paragraph{Proof of Theorem \ref{thm. convergence}}

\begin{proof}
    Let ${L'}_i^k$ denote the $f$-divergence between the estimate of the expert occupancy measure and the estimate of the expert model right before the subsequent E-step, i.e.,
    \begin{align*}
        {L'}_i^k = \Df{\rho_{\pi^k,\Tx^k}}{\rho_{E}^{k-1}}{\oaxx{i}} .
    \end{align*}
    Then, we can derive the following inequalities regarding the loss updates:
    \begin{align*}
         &{L'}^k = \sum_i  {L'}_i^k \\
         &=\sum_i \left(\Df{\rho_{\pi^k,\Tx^k}}{\rho_{E}^{k-1}}{\oaxx{i}}\right)\\
         &=\sum_i \left( 
            \sum_{\oaxx{i}} p(x_i, \PrX{i}|o_i, a_i, \JoAM^{k-1})p_E(o_i, a_i) f \left( \frac{\rho_{\calN_i^k}(\oaxx{i})}{p(x_i, \PrX{i}|o_i, a_i, \JoAM^{k-1})p_E(o_i, a_i)} \right) \right) \\
         &\geq \sum_i \left( 
            \sum_{\obs_i, a_i} p_E(o_i, a_i) f \left( \frac{\rho_{\calN_i^k}(\obs_i, a_i)}{p_E(\obs_i, a_i)} \right)   \right)~\qquad ~(\because \text{\textit{f} is convex})\\
        &=\sum_i \left( 
            \sum_{\oaxx{i}} p(x_i, \PrX{i}|o_i, a_i, \JoAM^{k})p_E(o_i, a_i) f \left( \frac{\rho_{\calN_i^k}(\oaxx{i})}{p(x_i, \PrX{i}|o_i, a_i, \JoAM^{k})p_E(o_i, a_i)} \right)  \right)\quad (\because \text{E-step})\\
         &=\sum_i \left(\Df{\rho_{\pi^k,\Tx^k}}{\rho_{E}^{k}}{\oaxx{i}} \right)~~ (:= L^k)\\
         &\geq \sum_i \left(\Df{\rho_{\pi^{k+1},\Tx^{k+1}}}{\rho_{E}^{k}}{\oaxx{i}}\right) \\
         &\hspace{9cm}(\because \text{M-step with the condition (1)}) \\
         &={L'}^{k+1}
    \end{align*}
    Since ${L'}^k \geq {L}^k$ and ${L}^k \geq {L'}^{k+1}$, $L^{k} \geq L^{k=1}$.
\end{proof}


\subsection{E-step: Derivations of the MAP Estimation}
In E-step, we infer the values of unknown subtasks using the following MAP estimation:
\begin{align*}
\argmax_{\TmX{i}{0:h}} p(\TmX{i}{0:h}|\TmOA{}{0:h}, \PaAM{}{k}) &=\argmax_{\TmX{i}{0:h}} p(\TmX{i}{0:h}, \TmOA{}{0:h}| \PaAM{}{k}) \\
&= \argmax_{\TmX{i}{0:h}}\log p(\TmX{i}{0:h}, \TmOA{}{0:h}| \PaAM{}{k}) 
\end{align*}
This can be effectively computed using the following dynamic programming approach:
\begin{align*}
\alpha_t(x) &\doteq \max_{\TmX{i}{0:t-1}}\log p(x_i^t\myeq x, \TmX{i}{0:t-1}, \TmOA{}{0:t}| \PaAM{}{k})  + c_t\\
&=\max_{y}\left( \log \pi_i^k(a_i^t|\obs_i^t, x_i^t\myeq x) + \log \Tx_i^k(x_i^t\myeq x|\obs_i^t, x_i^{t-1}\myeq y) + \alpha_{t-1}(y)\right)\\
\alpha_0(x) &= \log \pi_i^k(a_i^0|\obs_i^0, x_i^t=x) + \log\Tx_i^k(x_i^0=x|\obs_i^0, \#)
\end{align*}
where $c_t$ is a constant that does not depend on $x$ and can be ignored for this MAP estimation.
Since the inference only depends on individual agent models, it can be decentralized. With a time complexity of $O(h|X_i|)$ for each agent $i$, the overall time complexity for $n$ agents is $O(nh|\bar{X}|^2)$ where $|\bar{X}|=\max_{i=1:n}|X_i|$.
\clearpage
\section{Experiment Details}
\subsection{Domain Descriptions}
\label{sec. domain descriptions }

\begin{wrapfigure}{r}{0.27\textwidth}
  \centering
  \includegraphics[width=0.27\textwidth]{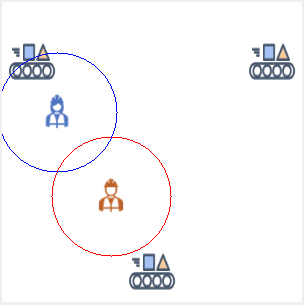}
  \caption{\simplemulti-$3$}
  \label{fig: mj3}
\end{wrapfigure}
\paragraph{\simplemulti-$n$} This domain gamifies the motivating example introduced in Fig. \ref{fig: workplace}. A team of two members must complete their job at $n$ designated locations. Each location has two different jobs for each agent, and an agent can complete only one type of job. If both agents are at the same location, they are distracted and none of them complete the job. A maximum of five jobs of each type can be stacked up per location, and agents can accomplish only one at a time. Once all five jobs are completed, the location is temporarily empty with that job type. However, after 15 timesteps, another five jobs regenerate at the location. While the world size is 10-by-10, agents can only observe within a radius of 2 (the red and blue circle around each agent in Fig. \ref{fig: mj3}). 
Additionally, they cannot know if a location has the jobs they are looking for unless they are at the location (distance threshold $< 0.5$). 
Both the observation and action spaces are continuous. An agent's observation consists of the following information: $(x, y, b_{job}, b_{tm},\Delta x_{tm}, \Delta y_{tm})$ where $(x, y)$ is the position of the agent, $b_{job}$ is a binary value indicating whether the agent is engaged in a job, $b_{tm}$ is a binary value indicating whether the teammate is observed, and $(\Delta x_{tm}, \Delta y_{tm})$ is the relative position of the teammate with respect to the agent. If the teammate is not within the observable distance, $b_{tm}, \Delta x_{tm}$ and $\Delta y_{tm}$ are all set to 0. We create synthetic agents that exhibit multimodal hierarchical behavior and generate 50 and 100 demonstrations for training and testing, respectively.

\paragraph{\movers and \rescue} These domains are employed from \cite{seo2023automated}. The goal of \movers is to move boxes to the truck with a team of two.  Since one agent cannot move boxes alone, explicit coordination between two agents is required. In \rescue, two agents are tasked with rescuing all the victims at four different disaster sites. While in City Hall and Campsite, the victims can be rescued by one agent, both agents must work together in bridges. In both domains, agents can observe only their vicinity (the unshaded area in Figs. \ref{fig: movers} and \ref{fig: flood}), forcing them to estimate their next location for effective collaboration. We assume an agent selects one of four locations (three boxes and the truck in \movers, and four disaster sites in \rescue) as their subtasks. In \movers, an agent's observation includes the following:  $(x, y, a, b_{tm},\Delta x_{tm},\Delta y_{tm}, a_{tm}, box_1, box_2, box_3)$ where $(x, y)$ represents the agent's position, $a$ is the agent's previous action, $b_{tm}$ is a binary indicating whether the teammate is observed, $(\Delta x_{tm}, \Delta y_{tm})$ represents the relative position of the teammate with respect to the agent, and $box_i$ represents the state of the box $i$. Each $box_i$ can take one of four values: ``Not observed'', ``At the original location'', ``Being carried'', and ``At the goal''. Since this domain is discrete, $x, y, a,  \Delta x_{tm}, \Delta y_{tm}, a_{tm}$ and $box_i$s are all one-hot encoded, respectively, resulting in the observation dimension of 45. When the teammate is not observed, $\Delta x_{tm}, \Delta y_{tm}$ and $a_{tm}$ are also set to 0 in addition to $b_{tm}$. In \rescue, an observation is represented by a 56-dimension vector: $(loc, a, b_{tm}, loc_{tm}, a_{tm}, w_1, w_2, w_3, w_4)$ where $loc$ is one of 32 positions on the map where the agent is currently located, $a$ is the agent's previous action, $b_{tm}$ is a binary indicating whether the teammate is observed, $loc_{tm}$ is the location of the teammate, $a_{tm}$ is the teammate's previous action, and $w_i$ is a binary indicating the status of the rescue work at each disaster site $i$. $loc_{tm}$ can take either "the same location as the agent" or one of six landmarks (``City Hall'', ``Fire Station'', ``Upper Bridge'', etc.). Similar to \movers, all variables are one-hot encoded, except that $loc_{tm}$ and $a_{tm}$ are set to all-zero when the teammate is not observed.
We use 50 and 100 demonstrations for training and testing per each domain.


\paragraph{\smactwo suite} \smactwo is a challenging benchmark for multi-agent reinforcement learning \cite{ellis2023smacv2}. We consider two domains in this suite: \protoss and \terran. The goal in these domains is to defeat an enemy team of five units(agents) by controlling five agents as a team. Each domain consists of different types of agents, and a team composition can differ between trials as they are randomly generated from a distribution.
We train a multi-agent policy via multi-agent reinforcement learning and generate 50 trajectories per domain for training. While the demonstrations do not include any diverse subtask-driven behavior, we set the number of subtasks as three for all agents to see the benefit of including latent states in the agent model.

\subsection{Expert Models}
We describe here the behavior of an expert team that has been used to generate demonstrations.
\paragraph{\simplemulti-$n$} We handcrafted subtask-driven expert behavior according to common sense rules. The set of subtasks consists of the locations designated for performing jobs (conveyor belts in Fig. \ref{fig: mj3}). Given a location as subtask, the action-level policy is implemented to move towards the location with some random noise. If the agent discovers that another agent is already at the targeted location, it waits around the target, maintaining some distance so as not to distract the other agent. Meanwhile, the subtask transition is modeled as follows: If the agent does not observe another agent and is not at the intended location, it maintains its subtask. If the agent does not gain a reward for conducting its job (due to no remaining jobs) at the intended location, it randomly selects one of the other locations as its new target. If the agent observes another agent, it randomly picks one of all locations as its new target. The agent maintains its subtask if it picks the same location it originally targeted. At the start of the task, each agent randomly selects one location as its subtask.

\paragraph{\movers and \rescue} We use value iteration to obtain experts' action-level policies. For the value iteration, we set a large positive reward for the states where an agent is at the intended location and a small negative reward for every other state as a penalty. The subtask transitions of each expert member are manually specified to follow the descriptions provided by \cite{seo2023automated}.

\paragraph{\smactwo suite} We use MAPPO \cite{yu2022the} to train a multi-agent policy, which is available at \url{https://github.com/marlbenchmark/on-policy}. We used the default set of hyperparameters and trained MAPPO with 20M exploration steps. It took a centralized approach, which shared policy networks across all agents.

\section{Implementation Details}
We use Python 3.8 to implement domains and algorithms and PyTorch 2.0.0 to build deep neural network models.

\subsection{Baseline Implementation}
We utilize official or popular source code as much as possible to implement reliable baselines. For Behavior Cloning, we use a version available at \url{https://github.com/HumanCompatibleAI/imitation} \cite{gleave2022imitation}. For \iql \cite{garg2021iq}, we use the official implementation provided by the authors at \url{https://github.com/Div99/IQ-Learn}. Note that while we do not directly use it as our baseline, it is essential for implementing both \iiqlsh and \ouralg. The implementation of \magail and \maogail are based on \ogail implementation that can be found at \url{https://github.com/id9502/Option-GAIL}. For \magail, we also make its critics dependent on other agents' information, following the original paper \cite{song2018multi}. 

For all baselines, we use a multi-layer perception (MLP) with two hidden 128-node layers for the actors and critics. For discriminators, we use MLP with two hidden layers of 256 nodes. The batch size has been standardized to 256.
For \movers and \rescue, we capped the maximum exploration steps at $100k$ for \iql-based approaches and $300k$ for GAIL-based approaches. For other domains, these limits were set to $200k$ and $500k$, respectively. BC was trained using $10k$ updates across all domains.

\subsection{\btil Adaptation}
Since \btil is limited to discrete observation and action spaces and cannot scale up to large domains, it is infeasible to be applied to \simplemulti-suite and \smactwo-suite. It is also impossible to apply \btil to   \movers and \rescue if we want to directly use the current observation representation, which is a concatenation of elements in subspaces. However, because the sets of observations in these domains are finite, we apply \btil by numbering each observation from 1 to $|\Omega_i|$. In this case, the number of possible observations for agent $i$ will be $\approx 1.0 \cdot 10^6$ for \movers and $\approx 1.3 \cdot 10^5$ for \rescue. For \btil, we use the implementation available at  \url{https://github.com/unhelkarlab/BTIL}\cite{seo2022semi}.

\subsection{Hyperparameters}
We grid-search the hyperparameter space to find the optimal one. In all domains, we use `Relu' as the neural net activation function, the learning rate of $3e-4$ for critics, the learning rate of $1e-4$ for actors, and a single critic structure. For the \iql submodule, we use $0.01$ for the temperature parameters, `value' for IQ method losses, and $Chi$ for $f$-divergence function. We use the replay buffer size $5k$ for \movers and \rescue, and $30k$ for other domains.

\subsection{Additional Results}
\subsubsection{Learning Curve}
We plot the task performance according to the number of exploration steps in Fig. \ref{fig: plots}. While \iql-based approaches generally outperformed GAIL-based approaches, as shown in Figs. \ref{fig: mj2 plot}-\ref{fig: movers plot}, \ouralg can additionally improve its performance with semi-supervision. In \rescue, overfitting is observed due to the low complexity of the domain.

\begin{figure}[t]
\newcommand\gap{0.30}
\newcommand\subgap{0.98}
  \centering
  \begin{subfigure}[b]{\linewidth}
      \centering
      \includegraphics[width=\subgap\textwidth]{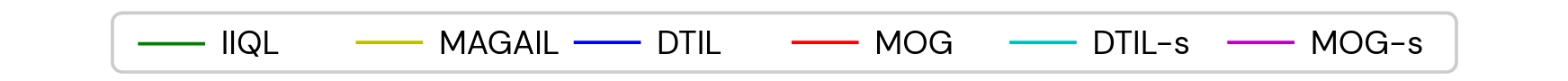}
  \end{subfigure}
  \begin{subfigure}[b]{\gap\linewidth}
      \centering
      \includegraphics[width=\subgap\textwidth]{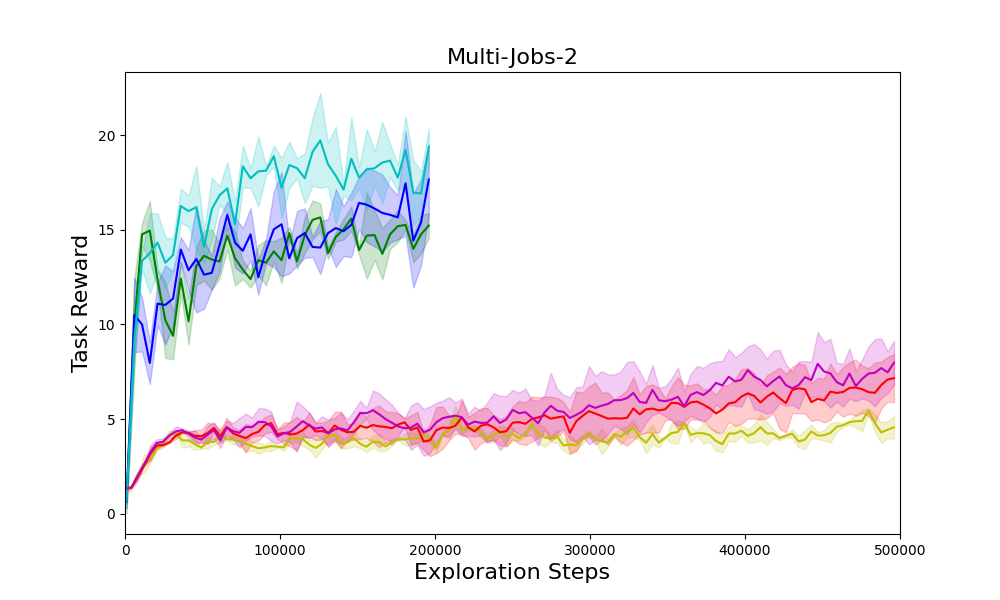}
      \caption{\simplemulti-$2$}
      \label{fig: mj2 plot}
  \end{subfigure}
  \hfill  
  \begin{subfigure}[b]{\gap\linewidth}
      \centering
      \includegraphics[width=\subgap\textwidth]{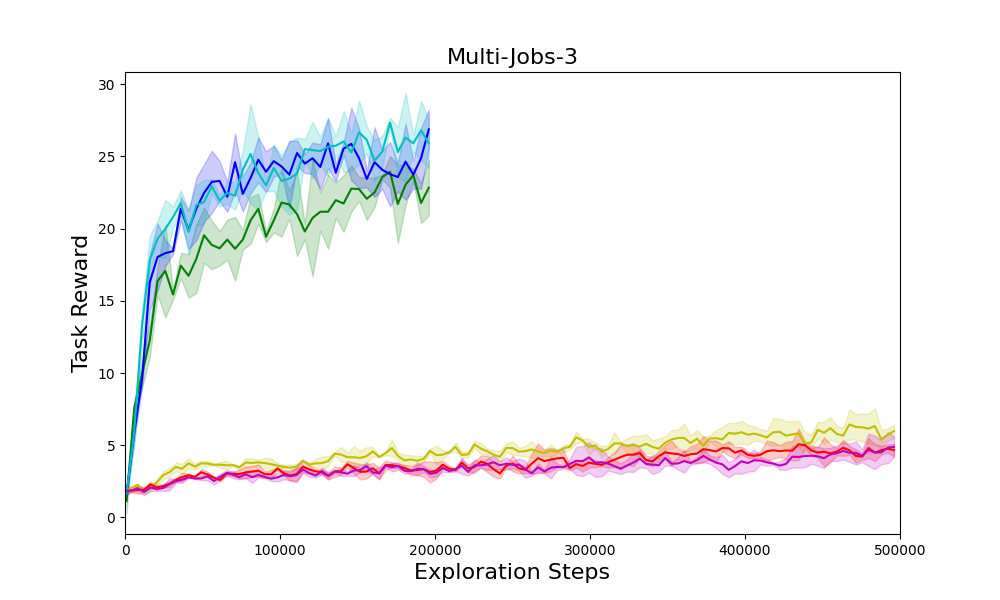}
      \caption{\simplemulti-$3$}
      \label{fig: mj3 plot}
  \end{subfigure}
  \hfill  
  \begin{subfigure}[b]{\gap\linewidth}
      \centering
      \includegraphics[width=\subgap\textwidth]{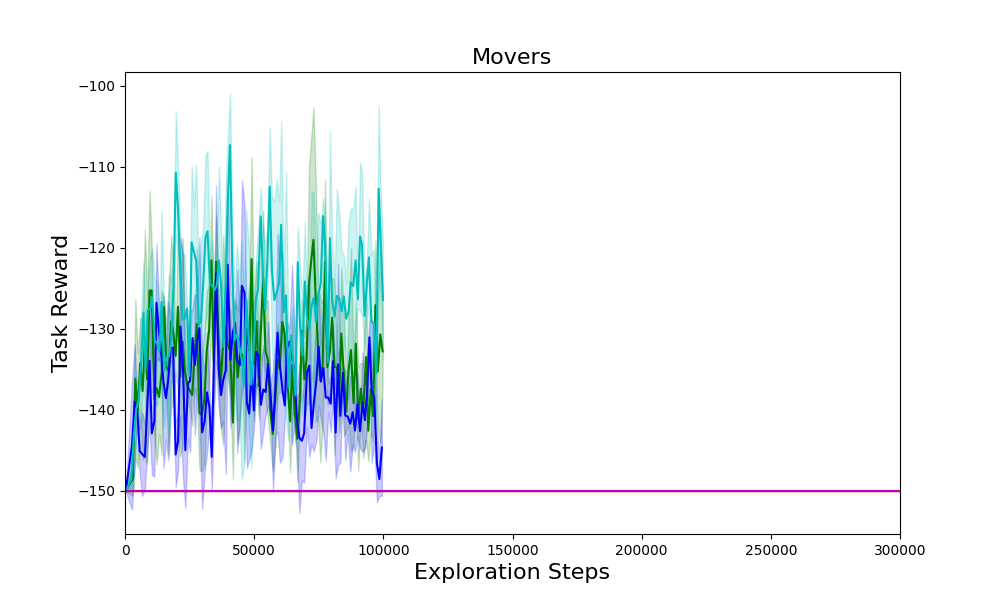}
      \caption{\movers}
      \label{fig: movers plot}
  \end{subfigure}
  \hfill  
  \begin{subfigure}[b]{\gap\linewidth}
      \centering
      \includegraphics[width=\subgap\textwidth]{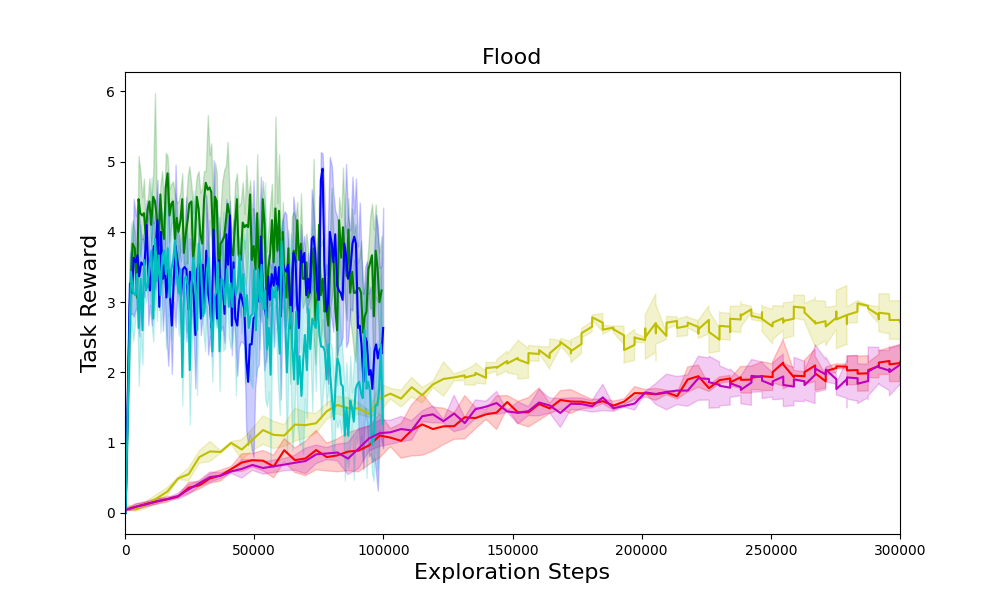}
      \caption{\rescue}
      \label{fig: flood plot}
  \end{subfigure}
  \hfill  
  \begin{subfigure}[b]{\gap\linewidth}
      \centering
      \includegraphics[width=\subgap\textwidth]{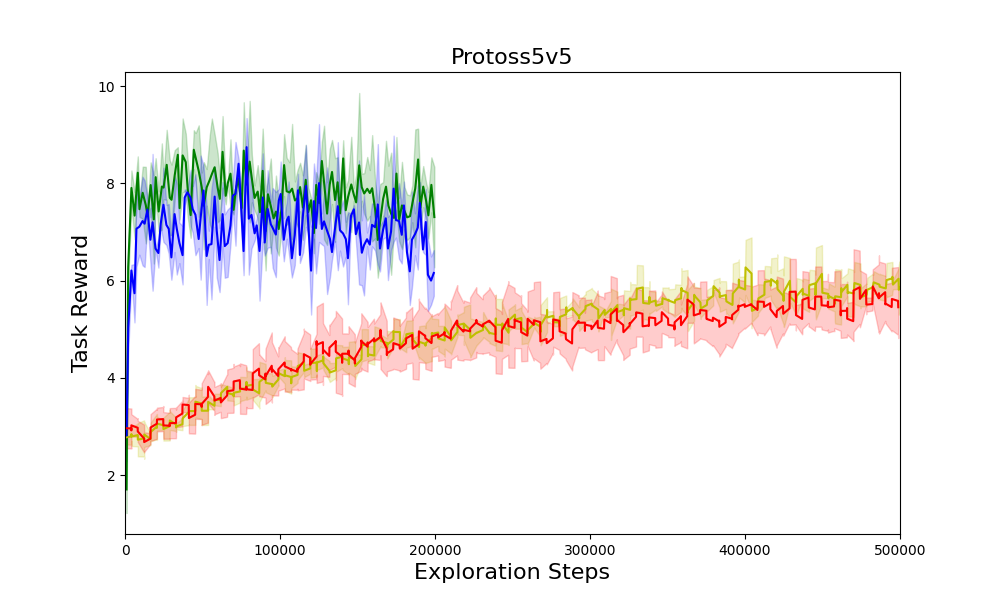}
      \caption{\protoss}
      \label{fig: protoss plot}
  \end{subfigure}
  \hfill  
  \begin{subfigure}[b]{\gap\linewidth}
      \centering
      \includegraphics[width=\subgap\textwidth]{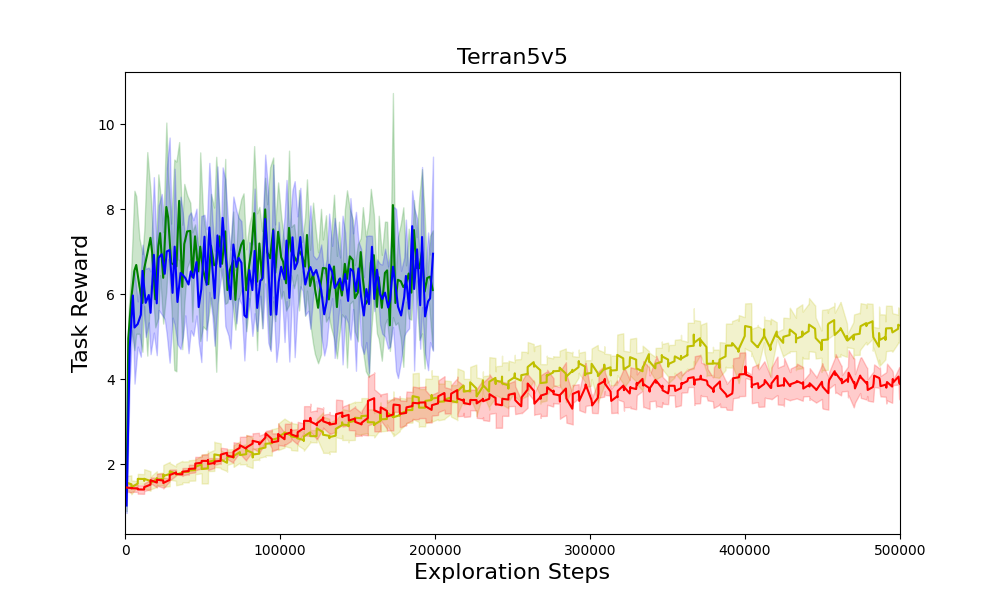}
      \caption{\terran}
      \label{fig: terran plot}
  \end{subfigure}

  \captionsetup{subrefformat=parens}
  \caption{The average task returns vs. the number of exploration steps. Each method is plotted with three seeds.}
  \label{fig: plots}
\end{figure}

\subsubsection{Behavior of team models in unsupervised learning settings.}

Without supervision, a hierarchical imitation learner may develop entirely different action policies that do not align with any subtask-driven behaviors of the expert. Nevertheless, even in such settings, our approach outperforms \maogail as a multi-agent hierarchical imitation learner. Since no clear metric exists for quantifying subtask-driven behavior quality, we qualitatively assess this advantage by visualizing the learned team models' paths for each $x$. Fig. \ref{fig: unsupervised paths} shows team models learned without subtask annotations. As illustrated, the behavior differences based on $x$ are more discernible in \ouralg than \maogail. Unlike \maogail, \ouralg more clearly exhibits subtask-driven behaviors, approaching different conveyor belts based on $x$. Notably, the learned $x$ values do not necessarily correspond to the expert's actual subtasks due to a lack of grounding. For instance, the behaviors for $x=2$ and $x=3$ are swapped relative to the expert's substask indices $2$ and $3$, as presented in Figure \ref{fig: expert a1 visualization}.

\begin{figure}[h]
  \def\subfwid{.48}
  \def\figwid{.3}
  \centering
  \begin{subfigure}[t]{\subfwid\linewidth}
      \centering
      \includegraphics[width=\figwid\linewidth, frame]{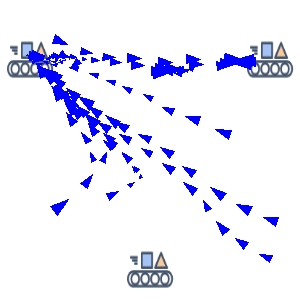}\hspace{0.5ex}
      \includegraphics[width=\figwid\linewidth, frame]{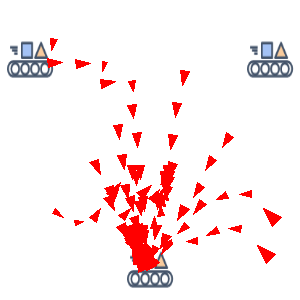}\hspace{0.5ex}
      \includegraphics[width=\figwid\linewidth, frame]{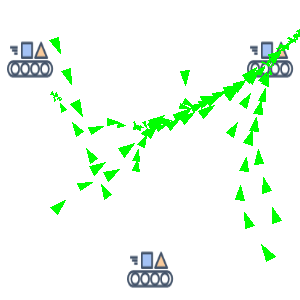}
      \caption{\ouralg (Agent 1)}
      \label{fig: unsupervised DTIL a1}
  \end{subfigure}
  \hspace{1ex}
  \begin{subfigure}[t]{\subfwid\linewidth}
      \centering
      \includegraphics[width=\figwid\linewidth, frame]{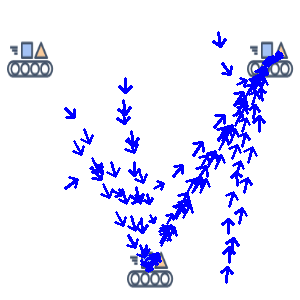}\hspace{0.5ex}
      \includegraphics[width=\figwid\linewidth, frame]{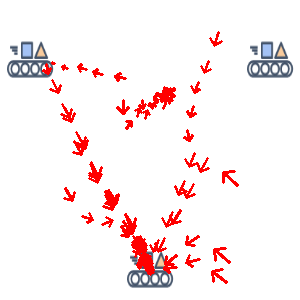}\hspace{0.5ex}
      \includegraphics[width=\figwid\linewidth, frame]{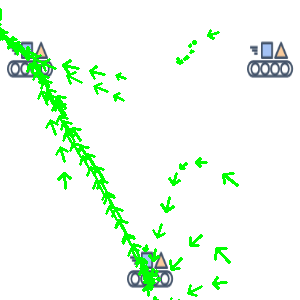}
      \caption{\ouralg (Agent 2)}
  \end{subfigure}
  \begin{subfigure}[t]{\subfwid\linewidth}
      \centering
      \includegraphics[width=\figwid\linewidth, frame]{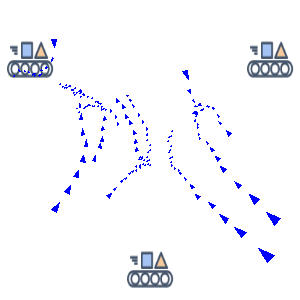}\hspace{0.5ex}
      \includegraphics[width=\figwid\linewidth, frame]{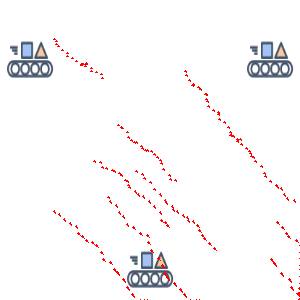}\hspace{0.5ex}
      \includegraphics[width=\figwid\linewidth, frame]{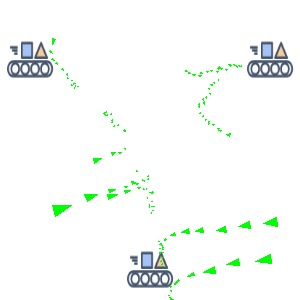}
      \caption{\maogail (Agent 1)}
  \end{subfigure}
  \hspace{1ex}
  \begin{subfigure}[t]{\subfwid\linewidth}
      \centering
      \includegraphics[width=\figwid\linewidth, frame]{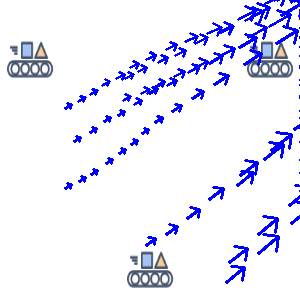}\hspace{0.5ex}
      \includegraphics[width=\figwid\linewidth, frame]{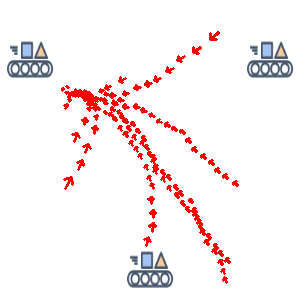}\hspace{0.5ex}
      \includegraphics[width=\figwid\linewidth, frame]{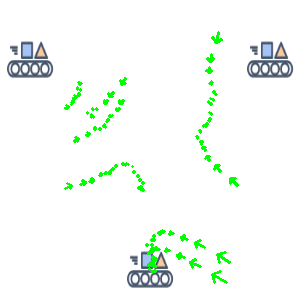}
      \caption{\maogail (Agent 2)}
  \end{subfigure}
  \captionsetup{subrefformat=parens}
  \caption{Individual \simplemulti-$3$ trajectories for each $x$ generated by team models learned without any subtask annotations.}
  \label{fig: unsupervised paths}
\end{figure}

\subsubsection{Visualization of diverse team behaviors learned with \ouralg.}

Figure \ref{fig: paths} illustrates the diverse team behaviors that can emerge in the same situation due to different subtask updates. To simulate team interactions, we manually set the agents' starting positions and assign their subtasks to ensure they cross paths at the center of the space. 
This scenario is simulated across five different seeds ($seed=0:4$). The dashed boxes in each image indicate instances where the two agents come close enough to observe each other.  As shown in Fig. \ref{fig: expert interaction paths}, the expert team exhibits diverse behaviors in this interactive setting. While \maogail agents produce only similar trajectories, \ouralg emulates diverse team trajectories, suggesting that it effectively learns subtask-driven agent models from heterogeneous multi-agent demonstrations.

\begin{figure}[h]
  \def\subfwid{.80}
  \def\figwid{.18}
  \centering
  \begin{subfigure}[t]{\subfwid\linewidth}
      \centering
      \includegraphics[width=\figwid\linewidth, frame]{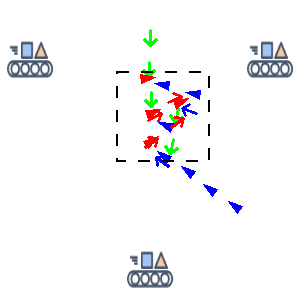}\hspace{0.5ex}
      \includegraphics[width=\figwid\linewidth, frame]{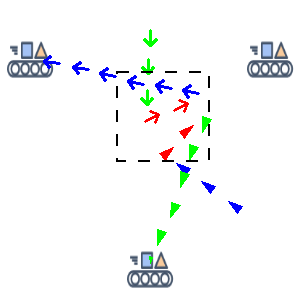}\hspace{0.5ex}
      \includegraphics[width=\figwid\linewidth, frame]{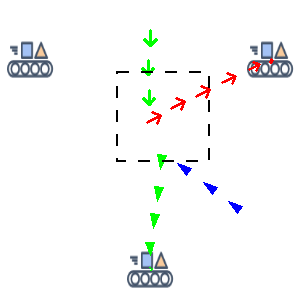}\hspace{0.5ex}
      \includegraphics[width=\figwid\linewidth, frame]{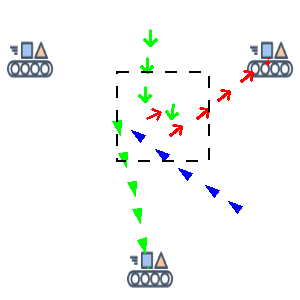}\hspace{0.5ex}
      \includegraphics[width=\figwid\linewidth, frame]{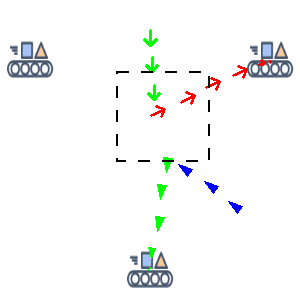}
      \caption{Expert}
      \label{fig: expert interaction paths}
  \end{subfigure}
  \begin{subfigure}[t]{\subfwid\linewidth}
      \centering
      \includegraphics[width=\figwid\linewidth, frame]{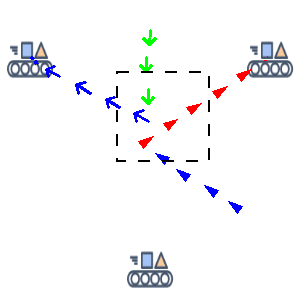}\hspace{0.5ex}
      \includegraphics[width=\figwid\linewidth, frame]{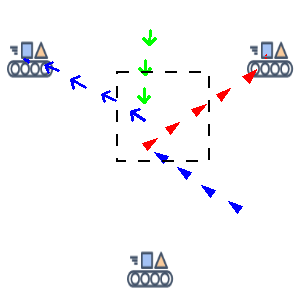}\hspace{0.5ex}
      \includegraphics[width=\figwid\linewidth, frame]{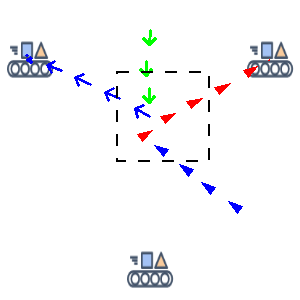}\hspace{0.5ex}
      \includegraphics[width=\figwid\linewidth, frame]{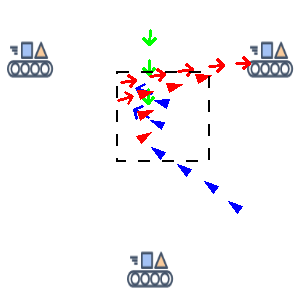}\hspace{0.5ex}
      \includegraphics[width=\figwid\linewidth, frame]{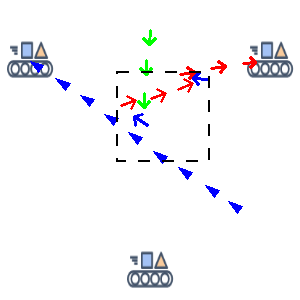}
      \caption{\ouralg}
  \end{subfigure}
  \begin{subfigure}[t]{\subfwid\linewidth}
      \centering
      \includegraphics[width=\figwid\linewidth, frame]{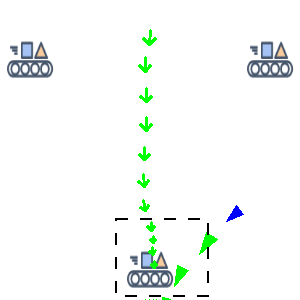}\hspace{0.5ex}
      \includegraphics[width=\figwid\linewidth, frame]{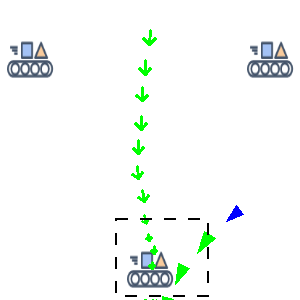}\hspace{0.5ex}
      \includegraphics[width=\figwid\linewidth, frame]{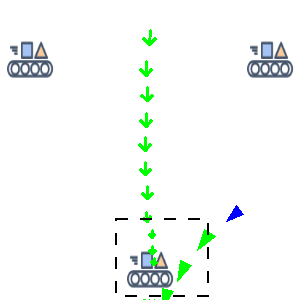}\hspace{0.5ex}
      \includegraphics[width=\figwid\linewidth, frame]{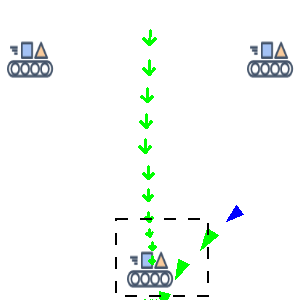}\hspace{0.5ex}
      \includegraphics[width=\figwid\linewidth, frame]{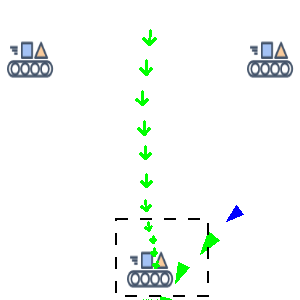}
      \caption{\maogail}
  \end{subfigure}
  \captionsetup{subrefformat=parens}
  \caption{Example trajectories of \simplemulti-$3$ trajectories generated by the expert and learned models (20 \% supervision). The triangles and arrows represent the actions of Agent 1 and Agent 2, respectively. The dashed boxes represent the locations where two agents observed each other. }
  \label{fig: paths}
\end{figure}

\fi

\end{document}